\documentclass{article}

\usepackage[preprint]{neurips_2019}

\usepackage[utf8]{inputenc} 
\usepackage[T1]{fontenc}    
\usepackage{hyperref}       
\usepackage{url}            
\usepackage{booktabs}       
\usepackage{amsfonts}       
\usepackage{nicefrac}       
\usepackage{microtype}      
\usepackage{latexsym,amsmath,amssymb,amsthm,eucal,bbm,color}
\usepackage{url}
\usepackage{comment}
\usepackage{float}
\usepackage{tcolorbox}
\usepackage{booktabs}
\usepackage{blindtext}
\usepackage{hyperref}
\hypersetup{
    colorlinks,
    linkcolor={red!50!black},
    citecolor={blue!50!black},
    urlcolor={blue!80!black}
}
\usepackage{multirow}
\usepackage{booktabs}
\usepackage{array}
\usepackage{soul}
 
\usepackage[normalem]{ulem}
\usepackage{stackengine}
\usepackage{parskip}
\usepackage{bm}
\usepackage{balance}
\usepackage{comment}
\usepackage{soul}
\usepackage{listings}
\usepackage[version=4]{mhchem}
\usepackage{empheq}
\usepackage{mdframed}
\usepackage{nomencl,etoolbox,ragged2e,siunitx}
\usepackage{tikz}
\usetikzlibrary{shapes,arrows}
\usetikzlibrary{er,positioning}

\usetikzlibrary{fit,positioning}
\tikzstyle{block} = [rectangle, draw, fill=blue!20, 
    text width=12.8em, text centered, rounded corners, minimum height=4em]
\tikzstyle{line} = [draw, -latex']
\tikzstyle{cloud} = [draw, ellipse,fill=red!20, node distance=3cm,
    minimum height=2em]
\usepackage{mdframed}

\newtheorem{thm}{Theorem}

\newtheorem{prop}{Proposition}
\newtheorem{cor}{Corollary}
\newtheorem{lemma}{Lemma}
\newtheorem{defn}{Definition}

\newmdtheoremenv{mhyp}{Hypothesis} 
\newmdtheoremenv{mthm}{Theorem}
\newmdtheoremenv{mtheorem}{Theorem}
\newmdtheoremenv{mprop}{Proposition}
\newmdtheoremenv{mcor}{Corollary}
\newmdtheoremenv{mlemma}{Lemma}
\newmdtheoremenv{mdefn}{Definition}
\newmdtheoremenv{mmydef}{Definition}
\newmdtheoremenv{mconj}{Conjecture}
\newmdtheoremenv{mex}{Example}
\newmdtheoremenv{mexercise}{Exercise}
\usepackage{wrapfig}
\usepackage{caption}

\usepackage{tikz}
\usetikzlibrary{calc}
\usepackage{pgfplots}
\usepackage{mwe}
\usepackage[]{algorithm2e}
\usepackage{cancel}
\DeclareMathAlphabet\mathbfcal{OMS}{cmsy}{b}{n}
\DeclareMathOperator*{\argmax}{arg\,max}
\DeclareMathOperator*{\argmin}{arg\,min}

\def \P{\mathbf{P}}

\def \PD{\text{PD}}
\newcommand{\bigcdot}{\boldsymbol{\cdot}}
\allowdisplaybreaks

\def \bx{\boldsymbol{x}}

\def \by{\boldsymbol{y}}
\def \bz{\boldsymbol{z}}

\def \bh{\boldsymbol{h}}

\def \bb{b}

\def \br{\boldsymbol{r}}

\def \bA{\boldsymbol{A}}

\def \bB{\boldsymbol{B}}

\def \bC{\boldsymbol{C}}

\def \bW{W}

\def\R{{\mathbb R}}

\def \Epsilon{\mathcal{E}}
\def \EPSILON{E}

\newcommand{\qq}{\vspace*{-2mm}}

\newcommand {\richb}[1]{{\color{blue}\sf{[richb: #1]}}}
\newcommand  {\randall}[1]{{\color{red}\sf{[rb: #1]}}}

\usepackage{xcolor}

\title{
The Geometry of Deep Networks: \\ Power Diagram Subdivision
}

\author{%
  Randall Balestriero\\
  ECE Department\\
  Rice University\\
   \And
  Romain Cosentino\\
  ECE Department\\
  Rice University\\
   \And
  Behnaam Aazhang\\
  ECE Department\\
  Rice University\\
   \And
  Richard Baraniuk\\
  ECE Department\\
  Rice University\\
}

\begin{document}

\maketitle

\begin{abstract}
\noindent
We study the geometry of deep (neural) networks (DNs) with piecewise affine and convex nonlinearities. 
The layers of such DNs have been shown to be {\em max-affine spline operators} (MASOs) that partition their input space and apply a region-dependent affine mapping to their input to produce their output.
We demonstrate that each MASO layer's input space partitioning corresponds to a {\em power diagram} (an extension of the classical Voronoi tiling) with a number of regions that grows exponentially with respect to the number of units (neurons).
We further show that a composition of MASO layers (e.g., the entire DN) produces a progressively subdivided power diagram and provide its analytical form. 
The subdivision process constrains the affine maps on the (exponentially many) power diagram regions to greatly reduce their complexity.
For classification problems, we obtain a formula for a MASO DN's decision boundary in the input space plus a measure of its curvature that depends on the DN's nonlinearities, weights, and architecture. 
Numerous numerical experiments support and extend our theoretical results.

\end{abstract}

\section{Introduction}
\label{sec:intro}

Deep learning has significantly advanced our ability to address
a wide range of difficult machine learning and signal
processing problems. Today’s machine learning landscape
is dominated by deep (neural) networks (DNs), which are
compositions of a large number of simple parameterized
linear and nonlinear transformations. Deep networks perform
surprisingly well in a host of applications; however,
surprisingly little is known about why or how they work so
well.

Recently, \cite{reportRB,balestriero2018spline} connected a
large class of DNs to a special kind of spline, which enables
one to view and analyze the inner workings of a DN using
tools from approximation theory and functional analysis.
In particular, when the DN is constructed using convex
and piecewise affine nonlinearities (such as ReLU, leaky-
ReLU, max-pooling, etc.), then its layers can be written
as {\em Max-Affine Spline Operators} (MASOs). An important
consequence for DNs is that each layer partitions its input
space into a set of regions and then processes inputs via a
simple affine transformation that changes from region to
region. 
{\em Understanding the geometry of the layer partition
regions – and how the layer partition regions combine into
a global input partition for the entire DN – is thus key to
understanding the operation of DNs.}

There has only been limited work in the geometry of deep
networks. The originating MASO work of \cite{reportRB,balestriero2018spline} focused on the analytical form of the
region-dependent affine maps and empirical statistics on
the partition. The work of \cite{wang2018a} empirically
studied this partitioning highlighting the fact that knowledge
of the DN partitioning alone is sufficient to reach high
performance. Other works have focused on the properties
of the partitioning, such as upper bounding the number of
regions \cite{montufar2014number,raghu2017expressive,hanin2019complexity}. An explicit
characterization of the input space partitioning of one
hidden layer DNs with ReLU activation has been proposed
in \cite{zhang2016understanding} by means of tropical geometry.

In this paper, we adopt a computational and combinatorial
geometry \cite{pach2011combinatorial,preparata2012computational} perspective of MASO-based DNs to derive the analytical
form of the input-space partition of a DN unit, a DN
layer, and an entire end-to-end DN. We demonstrate that
each MASO DN layer partitions its input feature map space partitioning according to a {\em power
diagram} (PD) (also known as a Laguerre–Voronoi diagram) \cite{aurenhammer1988geometric} with an exponentially large number
of regions.
Furthermore, the composition of the several MASO layers comprising a DN effects a {\em subdivision} process that creates the overall DN input-space partition.

Our complete, analytical characterization of the input-space and feature map
partition of MASO DNs opens up new avenues to study the
geometrical mechanisms behind their operation.

We summarize our contributions, which apply to any DN
employing piecewise affine and convex nonlinearities such
as fully connected, convolutional, with residual connections:
\begin{enumerate}

    \item We demonstrate that a DN partitions its input feature map space according to a {\em PD subdivision} (Sections~\ref{sec:layer},~\ref{sec:DEEP}). 
    We derive the analytical formula for a DN's PDs and point out their most interesting geometrical properties.

    \item We study the  computational and combinatorial geometric properties of the layer and DN partitioning (Section~\ref{sec:cell_properties}). In particular, a DN can infers the PD region to which any input belongs with a computational complexity that is asymptotically logarithmic in the number of regions.

    \item We demonstrate how the centroids of the layer PDs can be efficiently computed via backpropagation (Section~\ref{sec:computation}), which permits ready visualization of a PD.

    \item In the classification setting, we derive an analytical formula for the DN's decision boundary in term of the DN input space partitioning (Section~\ref{sec:boundary}). 
    The analytical formula enables us to characterize certain geometrical properties of the boundary.
    
\end{enumerate}

Additional background information plus proofs of the main
results are provided in several appendices. 

\qq
\section{Background on Deep Networks and Max-Affine Spline Operators}
\label{sec:back}
\qq

A deep network (DN) is an operator $f_\Theta$ with parameters $\Theta$ that maps an input signal $\bx\in\R^D$ to the output prediction $\widehat{y}\in \R^C$. Current DNs can be written as a  composition of $L$ intermediate {\em layer} mappings $f^{(\ell)}: \mathcal{X}^{(\ell-1)} \rightarrow \mathcal{X}^{(\ell)}$ ($\ell=1,\dots, L$) with $\mathcal{X}^{(\ell)}\subset \mathbb{R}^{D(\ell)}$ that transform an input {\em feature map} $\bz^{(\ell-1)}$ into the output feature map $\bz^{(\ell)}$ with the initializations $\bz^{(0)}(x):=x$ and $D(0)=D$. The feature maps $\bz^{(\ell)}$ can be viewed equivalently as signals, flattened vectors, or tensors. 

DN layers can be constructed from a range of different linear and nonlinear operators. 
One important linear operator is the \textbf{{\em fully connected operator}} that performs an arbitrary affine transformation by multiplying its input by the dense matrix $\bW^{(\ell)} \in \mathbb{R}^{D(\ell) \times D(\ell-1)}$ and adding the arbitrary bias vector $\bb_\bW^{(\ell)} \in \mathbb{R}^{D(\ell)}$ as in $f^{(\ell)}_\bW \!\left(\bz^{(\ell-1)}(\bx)\right):=   \bW^{(\ell)}\bz^{(\ell-1)}(\bx)+\bb_\bW^{(\ell)}$. Further examples are provided in \cite{goodfellow2016deep}. 
Given the collection of linear and nonlinear operators making up a DN, the following definition yields a single, unique layer decomposition. 

\medskip
\begin{defn}
A DN {\bf layer} $f^{(\ell)}$ comprises a single nonlinear DN operator composed with any (if any) preceding linear operators that lie between it and the preceding nonlinear operator.
\label{def:layer}
\end{defn}

Work from \cite{reportRB,balestriero2018spline} connects DN layers with {\em max-affine spline operators} (MASOs) . 
A MASO is an operator $S[A,B]:\mathbb{R}^{D}\rightarrow \mathbb{R}^{K}$ that 
concatenates $K$ independent {\em max-affine splines} \cite{magnani2009convex,hannah2013multivariate}, with each spline formed from $R$ affine mappings.
The MASO parameters consist of the ``slopes'' $A \in \mathbb{R}^{K \times R \times D}$ and the ``offsets/biases'' $B \in \mathbb{R}^{K\times R}$.\footnote{
The three subscripts of the slopes tensor $[A]_{k,r,d}$ correspond to output $k$, partition region $r$, and input signal index $d$.
The two subscripts of the offsets/biases tensor $[B]_{k,r}$ correspond to output $k$ and partition region $r$.
}
Given the input $\bx$, a MASO produces the output $\bz$ via 
\begin{align}
[\bz]_k\hspace{-0.1cm}
&=
\left[S[A,B](\bx)\right]_k=\hspace{-0.02cm}
\max_{r} \left(\left\langle [A]_{k,r,\bigcdot}, \bx \right\rangle\hspace{-0.08cm}+\hspace{-0.08cm}[B]_{k,r} \right),
\label{eq:MASO}
\end{align}
where $[\bz]_k$ denotes the $k^{\rm th}$ dimension of $\bz$.
The key background result for this paper is that a very large class of DNs are constructed from MASOs layers.

\medskip
\begin{thm}
Any DN layer $f^{(\ell)}$ constructed from operators that are piecewise-affine and convex can be written as a MASO with parameters $A^{(\ell)},B^{(\ell)}$ and output dimension $K=D(\ell)$.
Hence, a DN is a composition of $L$ MASOs \cite{reportRB,balestriero2018spline}.
\label{thm:MASO}
\end{thm}

For example, a layer made of a fully connected operator followed by a leaky-ReLU with leakiness $\eta$ has parameters $[A^{(\ell)}]_{k,1,\bigcdot}=[W^{(\ell)}]_{k,\bigcdot},[A^{(\ell)}]_{k,2,\bigcdot}=\eta [W^{(\ell)}]_{k,\bigcdot}$ for the slope parameter and $[B^{(\ell)}]_{k,1,\bigcdot}=[b^{(\ell)}]_{k},[B^{(\ell)}]_{k,2}=\eta [b^{(\ell)}]_{k}$ for the bias.
%
%
A DN comprising $L$ MASO layers is a continuous affine spline operator with an input space partition and a partition-region-dependent affine mapping.
%
{\em However, little is known analytically about the input-space partition.}

{\em This paper characterizes the geometry of the MASO partitions of the input space and the feature map spaces $\mathcal{X}^{(\ell)}$.}
We proceed by first studying the geometry of a single layer (Section \ref{sec:layer}) and then the composition of $L$ layers that forms a complete DN (Section \ref{sec:DEEP}).
Voronoi diagrams and their generalization, Power diagrams, play a key r\^{o}le in our analysis, and we turn to these next.

\section{Background on Voronoi and Power Diagrams}

A {\em power diagram} (PD), also known as a Laguerre–Voronoi diagram
\cite{aurenhammer1988geometric}, is a generalization of the classical Voronoi diagram.
\medskip

\begin{defn}
A PD partitions a space $\mathcal{X}$ into $R$ disjoint regions $\Omega=\{\omega_1,\dots,\omega_R\}$ such that $\cup_{r=1}^R \omega_r=\mathcal{X}$, where each cell is obtained via $\omega_r = \{ \bx \in \mathcal{X} :r(\bx)=r\}, r=1,\dots,R$, with
\begin{align}
    r(\bx)=\argmin_{k=1,\dots,R} \|\bx - [\mu]_{k,\bigcdot} \|^2-[{\rm rad}]_k ,\label{eq:cell_pd}
\end{align}
\end{defn}

The parameter $[\mu]_{k,\bigcdot}$ is called the {\em centroid}, while $[{\rm rad}]_k$ is called the {\em radius}.
The PD is a generalization of the {\em Voronoi diagram} (VD) by introduction of the external radius term to the $\ell_2$ distance, leading to the {\em Laguerre distance} \cite{imai1985voronoi}.
See Figure~\ref{fig:PD} for two equivalent geometric interpretations of a PD.

In general, a PD is defined with nonnegative radii to provide additional geometric interpretations (see Appendix~\ref{appendix:extra_geometric}). However, the PD is the same under global shifting as $\argmin_k \|\bx - [\mu]_{k,\bigcdot} \|^2-[{\rm rad}]_k =\argmin_k \|\bx - [\mu]_{k,\bigcdot} \|^2-([{\rm rad}]_k +Q)$. Thus, we allow for arbitrary radius since it can always be shifted back to nonnegative by setting $Q=\min_k [{\rm rad}]_k$.
For additional geometric insights on VDs and PDs see \cite{preparata2012computational} and Appendix~\ref{appendix:extra_geometric}.


\begin{figure}[t!]
    \centering
    \begin{minipage}{0.55\linewidth}
    \centering
    \begin{minipage}{0.49\linewidth}
    \vspace{0.79cm}
    \includegraphics[width=1\linewidth]{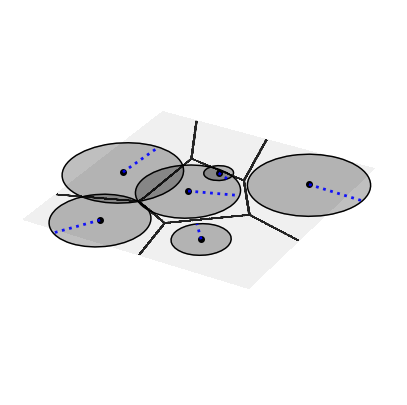}
    \end{minipage}
    \begin{minipage}{0.49\linewidth}
    \includegraphics[width=1\linewidth]{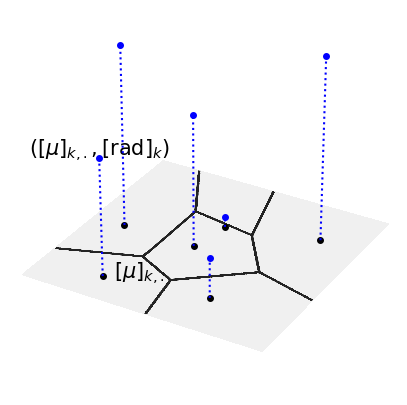}
    \end{minipage}
    \end{minipage}
    \begin{minipage}{0.44\linewidth}
    \caption{\small Two equivalent representations of a power diagram (PD).  {\bf Left:}~The grey circles have center $[\mu]_{k,\bigcdot}$ and radii $[{\rm rad}]_{k}$; each point $\bx$ is assigned to a specific region according to the Laguerre distance defined as the length of the segment tangent to and starting on the circle and reaching $\bx$.
    {\bf Right:}~
    A PD in $\mathbb{R}^D$ (here $D=2$) is constructed by lifting the centroids $[\mu]_{k,\bigcdot}$ up into an additional dimension in $\mathbb{R}^{D+1}$ by the distance $[{\rm rad}]_{k}$ and then finding the Voronoi diagram (VD) of the augmented centroids 
    $([\mu]_{k,\bigcdot},[{\rm rad}]_{k})$ in $\mathbb{R}^{D+1}$.
    The intersection of this higher-dimensional VD with the originating space $\mathbb{R}^D$ yields the PD.
    %
    \normalsize}
    \label{fig:PD}
    \end{minipage}
\end{figure}

\section{Input Space Power Diagram of a MASO Layer}
\label{sec:layer}

Like any spline, it is the interplay between the (affine) spline
mappings and the input space partition that work the magic
in a MASO DN. 
Indeed, the partition opens up new geometric
avenues to study how a MASO-based DN clusters and
organizes signals in a hierarchical fashion. 
However, little
is known analytically about the input-space partition other
than in the simplest case of a single unit with a constrained
bias value \cite{reportRB,balestriero2018spline}. 

We now embark on a programme to fully characterize the geometry of the input
space partition of a MASO-based DN. We will proceed in
three steps by studying the partition induced by i) one unit
of a single DN layer (Section~\ref{sec:unit}), ii) the combination of
all units in a single layer (Section~\ref{sec:layer}), iii) the composition of
L layers that forms the complete DN (Section~\ref{sec:net}).

\subsection{MAS Unit Power Diagram}
\label{sec:unit}


A MASO layer combines $K$ max affine spline (MAS) units $z_k$ to produce the layer output $\bz(\bx)=(z_1(\bx),\dots,z_K(\bx))$ given an input $\bx \in \mathcal{X}$. Denote each MAS computation from (\ref{eq:MASO}) as
\begin{align}
z_k(\bx)=\max_{r=1,\dots,R} \left\langle [A]_{k,r,\bigcdot},\, \bx \right\rangle +[B]_{k,r}=\max_{r=1,\dots,R} \Epsilon_{k,r}(\bx),
\label{eq:unit_formula}
\end{align}
where $\Epsilon_{k,r}(\bx)$ is the projection of $\bx$ onto the hyperplane parameterized by the slope $[A]_{k,r,\bigcdot}$ and offset $[B]_{k,r}$.
By defining the following half-space consisting of the set of points above the hyperplane 
\begin{equation}
    \Epsilon_{k,r}^{+} = \{(\bx,y) \in \mathcal{X}\times \mathbb{R} : y \geq \Epsilon_{k,r}(\bx) \},
    \label{eq:epsilon_unit}
\end{equation}
we obtain the following geometric interpretation of the unit output.

\begin{prop}
The $k^{\text{th}}$ MAS unit maps its input space onto the boundary of the convex polytope  $\mathcal{P}_k = \cap_{r=1}^R \Epsilon_{k,r}^{+}$, leading to
\begin{equation}
    \mathcal{X} \times z_k(\mathcal{X}) = \partial \mathcal{P}_k
    \label{eq:polytope}
\end{equation}
where we remind that $z_k(\mathcal{X})={\rm Im}(z_k) = \{ z_k(\bx), \bx \in \mathcal{X} \}$ is the image of $z_k$.
\label{prop:singleunit}
\end{prop}
To provide further intuition, we highlight the role of $\mathcal{P}_k$ in term of input space partitioning.

\begin{lemma}
The vertical projection on the input space $\mathcal{X}$ of the faces of the polytope $\mathcal{P}_k$ from (\ref{eq:polytope}) define the cells of a PD.
\label{prop:vertical_unit}
\end{lemma}

Since the $k^{\text{th}}$ MAS unit projects an input $\bx$ onto the polytope face given by $r_k:\mathcal{X} \rightarrow  \{1,\dots,R\}$ (recall(\ref{eq:cell_pd})) corresponding to
\begin{equation}
r_k(\bx) = \argmax_{r=1,\dots,R} \Epsilon_{k,r}(\bx), \label{eq:r_unit}
\end{equation}
the collection of inputs having the same face allocation, defined as
$ \forall r \in \left \{1,\dots,R \right \}, \: \omega_r=\{\bx \in \mathcal{X}: r_k(\bx)=r\}$, constitutes the $r^{\text{th}}$ {\em partition cell} of the unit $k$ PD.

\begin{thm}
\label{thm:unitPD}
The $k^{\text{th}}$ MAS unit partitions its input space according to a PD with $R$ centroids given by 
$[\mu]_{k,r}=[A]_{k,r,\bigcdot}$, and 
$[{\rm rad}]_{k,r}=2[B]_{k,r}+\|[A]_{k,r,\bigcdot}\|^2,\forall r\in \{1,\dots,R\}$
(recall (\ref{eq:cell_pd})).
\end{thm} 

\begin{cor}
The input space partitioning of a DN unit is composed of convex polytopes.
\end{cor}


\subsection{MASO Layer Power Diagram}
\label{sec:layer_cell_encoding}
\label{sec:layer}

 
We study the layer case by studying the joint behavior of all the layer units.
A MASO layer is a continuous, piecewise affine operator made by the concatenation of $K$ MAS units (recall (\ref{eq:MASO})); we extend (\ref{eq:unit_formula}) to 
\begin{equation}
    \bz(\bx) = \big( \max_{r=1,\dots,R} \Epsilon_{1,r}(\bx),\dots,\max_{r=1,\dots,R} \Epsilon_{K,r}(\bx)\big), \forall \bx \in \mathcal{X}  \label{eq:Z_boundary}
\end{equation}
and the per unit face index function $r_k$ (\ref{eq:r_unit}) into the operator $\br:\mathcal{X} \rightarrow \{1,\dots,R \}^K$ defined as
\begin{equation}
    \br(\bx)=(\br_1(\bx),\dots,\br_K(\bx)).
\end{equation}
Following the geometric interpretation of the unit output from Proposition~\ref{prop:singleunit}, we extend (\ref{eq:epsilon_unit}) to 
\begin{align}
    E_{\br}^+=\{(\bx,\by) \in \mathcal{X}\times \mathbb{R}^K : [\by]_1 \geq \Epsilon_{1,[\br]_1}(\bx),\dots, [\by]_K \geq \Epsilon_{K,[\br]_K}(\bx)\},
    \label{eq:epsilon_layer}
\end{align}
where $[\br]_k$ is the $k^{th}$ component of the vector $\br(\bx)$.

\begin{prop}
The layer operator $\bz$ maps its input space into the boundary of the ${\rm dim}(\mathcal{X})+K$ dimensional convex polytope $\P = \cap_{\br \in \{1,\dots,R\}^K}E_{\br}^+$ as
\begin{equation}
    \mathcal{X} \times \bz(\mathcal{X})=\mathcal{X} \times \bz_1(\mathcal{X})\times \dots \times \bz_K(\mathcal{X}) = \partial \P. \label{eq:polytopelayer}
\end{equation}
\label{prop:singlelayer}
\end{prop}
%

Similarly to Proposition~\ref{prop:vertical_unit}, the polytope $\P$ is bound to the layer input space partitioning.
\begin{lemma}
The vertical projection on the input space $\mathcal{X}$ of the faces of the polytope $\P$ from Proposition~\ref{prop:singlelayer} define cells of a PD.
\label{lemma:vertical_layer}
\end{lemma}

The MASO layer projects an input $\bx$ onto the polytope face indexed by $\br(\bx) $ corresponding to
\begin{align}
\br(\bx) = (\argmax_{r=1,\dots,R} \Epsilon_{1,r}(\bx),\dots,\argmax_{r=1,\dots,R} \Epsilon_{K,r}(\bx)).
\end{align}
The collection of inputs having the same face allocation jointly across the $K$ units constitutes the $\br^{\text{th}}$ {\em partition cell} of the layer PD. 

\begin{thm}
\label{thm:layerPD}
DN layer partitions its input space according to a PD with
$\{1,\dots,R\}^K$ cells, centroids 
$\mu_{\br}=\sum_{k=1}^{K} [A]_{k,[\br]_{k},\bigcdot}$ and
radii ${\rm rad}_{\br}= 2\langle \textbf{1},B_{\br}\rangle+\left \| \mu_{\br}\right \|^2$
(recall (\ref{eq:cell_pd})).
\label{thm:projection_layer}
\end{thm}
\begin{cor}
The input space partitioning of a DN layer is composed of convex polytopes.
\end{cor}

\begin{figure}[t]
\begin{minipage}{0.5\linewidth}
\centering
\small
Rank 1 \hspace{0.5cm} Rank 2 Orthogonal \hspace{0.4cm} Rank 2\\ \normalsize
    \centering
    \vspace{-0.02cm}
    \includegraphics[width=0.32\linewidth]{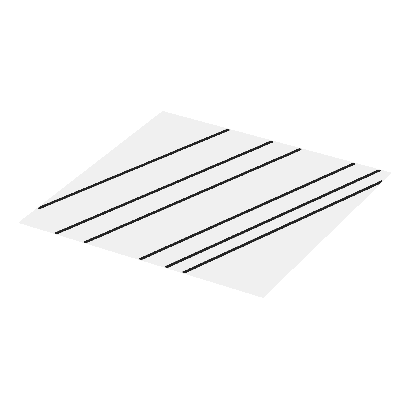}
    \includegraphics[width=0.32\linewidth]{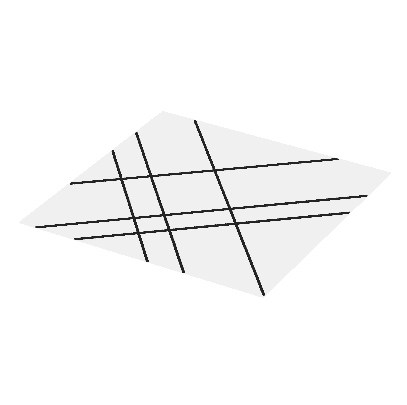}
    \includegraphics[width=0.32\linewidth]{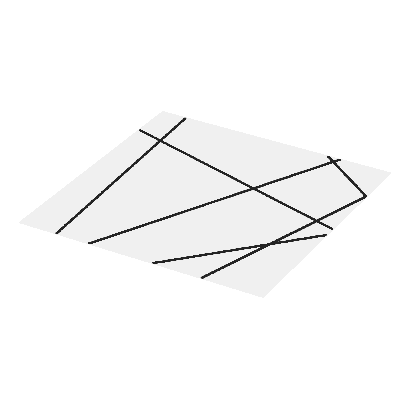}
\end{minipage}
\begin{minipage}{0.49\linewidth}
    \caption{\small
    Depiction of different types of weight constraints and their impact in the layer input space partitioning. On the left is depicted the case of a low rank matrix leading to colinear cuts, only the bias is responsible for shifting the cuts. In the middle orthogonal weights are used leading to orthogonal cuts. This, while not being degenerate still limits the input space partitioning. On the right, an arbitrary weight matrix is used leading the partitioning unconstrained.
    }
    \label{fig:toy_example}
    \end{minipage}
\end{figure}


\subsection{Weight Constraints and Cell Shapes}
\label{sec:cell_shape}
%
%
%
We highlight the close relationship between the layer weights $A,B$ from (\ref{eq:MASO}), the layer polytope $\P$ from Proposition~\ref{prop:singlelayer}, and the boundaries of the layer PD  from Theorem~\ref{thm:layerPD}. In particular, how one can alter or constraint the shape of the cells by constraining the weights of the layer.

{\bf Example 1:}
Constraining the layer weights to be such that $ [A]_{k,r,d} = 1_{\{d=[i]_{k,r}\}}[{\rm cst}]_{k,r}$ for some  integer $[i]_{k,r} \in \{1,\dots,D \}$, $D={\rm dim}(\mathcal{X})$, and arbitrary constant  $[{\rm cst}]_{k,r}$ leads to an input power diagram with cell boundaries parallel to the input space basis vectors see Fig.~\ref{fig:toy_example}.
For instance if the input space $\mathcal{X}$ is the Euclidean space $\mathbb{R}^D$ equipped with the canonical basis, the previous Proposition translates into having PD boundaries parallel to the axes.

{\bf Example 2:}
Constraining the layer weights to be such that $[A]_{k,r,d}=\pm [{\rm cst}]_{k,r}$ for some arbitrary constant  $[{\rm cst}]_{k,r}$ leads to a layer-input power diagram with diagonal cell boundaries.\footnote{
Note that while in example 1 each per unit $k$, per cell $r$ weight was constrained to contain a single nonzero element s.a. $(0, 0, c, 0)$ for $D=4$, example 2 makes the weight vector filled with a single constant but varying signs such as $(+c,-c,+c,-c)$.}

\begin{lemma}
Changing the radius of a given cell shrinks or expands w.r.t. the other \cite{aurenhammer1987power}.
\end{lemma}
\begin{thm}
Updating a single unit parameters (slope or offset of the affine transform and/or the nonlinearity behavior) affects multiple regions' centroids and radius.
\end{thm}
The above result recovers weight sharing concepts and implicit bias/regularization. In fact, most regions are tied together in term of learnable parameter. Trying to modify a single region while leaving everything else the same is not possible in general.

\section{Input Space Power Diagram of a MASO Deep Network}
\label{sec:DEEP}
\label{sec:net}

We consider the composition of multiple layers, as such, the input space of layer $\ell$ is denoted as $\mathcal{X}^{(\ell)}$, with $\mathcal{X}^{(0)}$ the DN input space. 

\subsection{
Power Diagram Subdivision}
\label{sec:layers}

\begin{figure}[t!]
    \centering
    \begin{minipage}{0.03\linewidth}
    \rotatebox{90}{\hspace{0.5cm}Input space partitioning\hspace{1.3cm}Partition polynomial\hspace{1.3cm}}
    \end{minipage}
    \begin{minipage}{0.95\linewidth}
    \centering
    \begin{minipage}{0.29\linewidth}
    \centering
    Layer 1: mapping\\$\mathcal{X}^{(0)}\subset \mathbb{R}^2$ to $\mathcal{X}^{(1)}\subset \mathbb{R}^6$
    \end{minipage}\hfill
    \begin{minipage}{0.29\linewidth}
    \centering
    Layer 2: mapping\\$\mathcal{X}^{(1)}\subset \mathbb{R}^6$ to $\mathcal{X}^{(2)}\subset \mathbb{R}^6$
    \end{minipage}\hfill
    \begin{minipage}{0.31\linewidth}
    \centering
    Layer 3 (classifier):  mapping $\mathcal{X}^{(2)}\subset \mathbb{R}^6$ to $\mathcal{X}^{(3)}\subset \mathbb{R}^1$
    \end{minipage}\\
    \begin{minipage}{0.32\linewidth}
    \centering
    \includegraphics[width=1\linewidth]{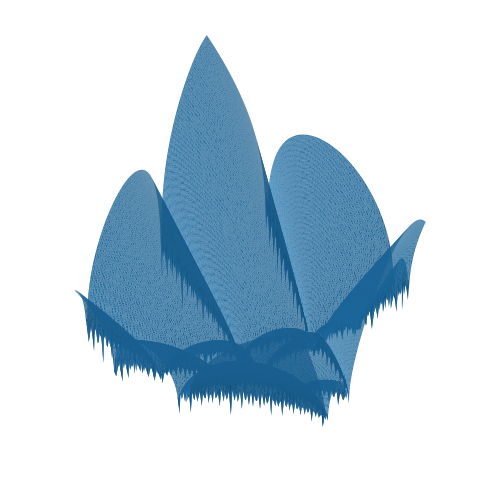}
    \end{minipage}\hfill
    \begin{minipage}{0.32\linewidth}
    \centering
    \includegraphics[width=1\linewidth]{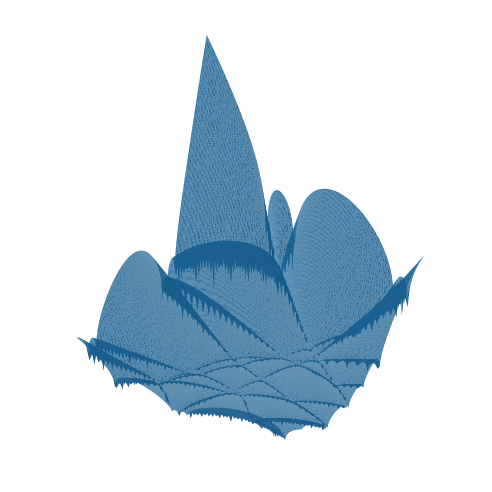}
    \end{minipage}\hfill
    \begin{minipage}{0.32\linewidth}
    \centering
    \includegraphics[width=1\linewidth]{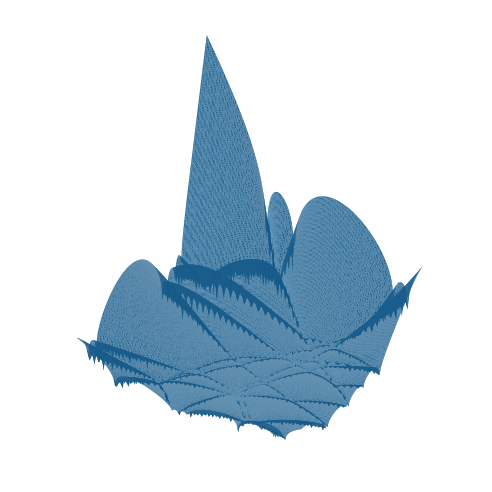}
    \end{minipage}\\
    \begin{minipage}{0.32\linewidth}
    \centering
    \includegraphics[width=1\linewidth]{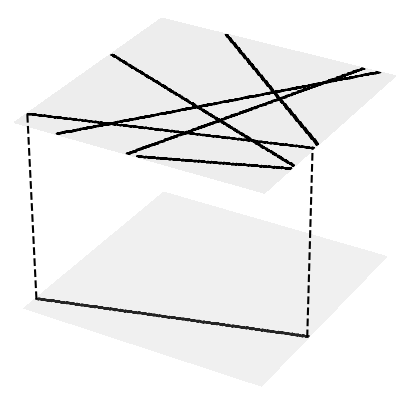}
    \end{minipage}
    \begin{minipage}{0.32\linewidth}
    \centering
    \includegraphics[width=1\linewidth]{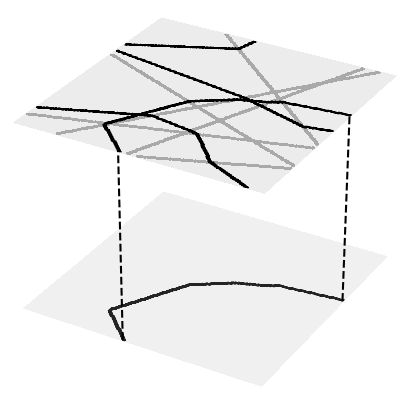}
    \end{minipage}
    \begin{minipage}{0.32\linewidth}
    \centering
    \includegraphics[width=1\linewidth]{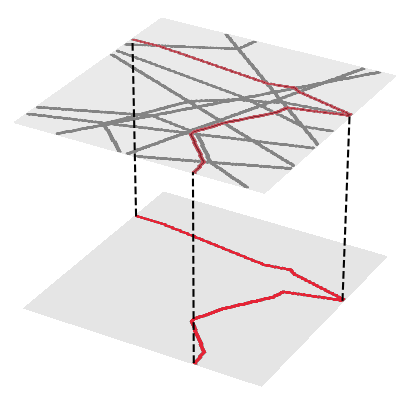}
    \end{minipage}
    \end{minipage}
    \caption{\small {\bf Top:}~The partition polynomial as defined in (\ref{eq:polynomial}), whose roots define the partition boundaries in the input space and determined by each layer parameters and nonlinearities.
    {\bf Bottom:}~Evolution of the input space partitioning layer after layer (left to right: $\Omega^{(1)}, \Omega^{(1,2)},\Omega^{(1,2,3)}$ from (\ref{eq:omega})) with the newly introduced boundaries in dark and previously built partitioning (being refined) in grey. Below each partitioning, one of the newly introduced cut denoted as ${\rm edge}_{\mathcal{X}^{(0)}}(k,\ell)$ from (\ref{eq:edge}) is highlighted, and which, in the last layer case (right), corresponds to the decision boundary (in red) see Figures~\ref{fig:additional_detail}, \ref{fig:additional_detail2} in Appendix~\ref{appendix:additional_figure} for additional examples.
    \normalsize}
    \label{fig:megaplot}
\end{figure}

We provide in this section the formula for deriving the input space partitioning of an $L$-layer DN by means of a recursive scheme.
Recall that each layer defines its own polytope $\P^{(\ell)}$ according to Proposition~\ref{prop:singlelayer}, each with domain $\mathcal{X}^{(\ell-1)}$.
The DN partitioning corresponds to a recursive subdivision where each per layer polytope subdivides the previously obtained partitioning, involving the representation of the considered layer polytope in the input space $\mathcal{X}^{(0)}$. This subdivision can be analytically derived from the following recursion scheme.
\\
{\bf Initialization}: The initialization consists of defining the part of the input space to consider $\mathcal{X}^{(0)}\subset \mathbb{R}^D$.
\\
{\bf Recursion step ($\ell=1$)}: The first layer subdivides $\mathcal{X}^{(0)}$  into a  $\PD$ from  Theorem~\ref{thm:layerPD} with parameters $A^{(1)}, B^{(1)}$ to obtain the layer 1 partitioning $\Omega^{(1)}$.
\\
{\bf Recursion step ($\ell=2$)}:
The second layer subdivides each cell of $\Omega^{(1)}$. Let's consider a specific cell $\omega^{(1)}_{\br^{(1)}}$; all inputs in this cell are projected to $\mathcal{X}^{(1)}$ by the first layer via $A^{(1)}_{\br^{(1)}}\bx +B^{(1)}_{\br^{(1)}}$.\footnote{Recall from (\ref{eq:MASO}) that $A^{(1)}_{\br^{(1)}},B^{(1)}_{\br^{(1)}}$ are the affine parameters associated to cell $\br^{(1)}$} The convex cell $\omega^{(1)}_{\br^{(1)}}$ thus remains a convex cell in $\mathcal{X}^{(1)}$ defined as the following affine transform of the cell
\begin{equation}
    {\rm aff}_{\br^{(1)}}= \{ A^{(1)}_{\br^{(1)}}\bx +B^{(1)}_{\br^{(1)}}, \bx \in \omega^{(1)}_{\br^{(1)}}\}\subset {\mathcal{X}^{(1)}}.\label{eq:affine}
\end{equation}
Since on the cell the first layer is linear; the slice of the polytope $\P^{(2)}\subset \mathcal{X}^{(1)}\times \mathbb{R}^{D(2)}$ (recall (\ref{eq:polytopelayer})) having for domain ${\rm aff}_{\br^{(1)}}$ formally defined as 
\begin{align}
\P^{(2)}_{\br^{(1)}} = \P^{(2)} \cap ({\rm aff}_{\br^{(1)}}\times \mathbb{R}^{D(2)}),\label{eq:restricted}
\end{align}
can thus be expressed w.r.t. $\mathcal{X}^{(0)}$. 
\begin{lemma}
The domain restricted polytope (\ref{eq:restricted}) can be expressed in the input space $\omega^{(1)}_{\br^{(1)}} \subset \mathcal{X}^{(0)}$ as
\begin{align}
    \P^{(1\leftarrow 2)}_{\br^{(1)}}\hspace{-0.13cm}=\hspace{-0.06cm} \cap_{\br^{(2)}} \{\hspace{-0.05cm}(\bx,\by) \hspace{-0.051cm}\in \hspace{-0.051cm}\omega^{(1)}_{\br^{(1)}} \hspace{-0.1cm} \times \mathbb{R}^{D(1)}\hspace{-0.1cm}:\hspace{-0.05cm}[\by]_1\hspace{-0.1cm}\geq \Epsilon^{(1\leftarrow 2)}_{1,[\br^{(2)}]_1}\hspace{-0.08cm}(\bx),\dots,[\by]_{D(1)}\hspace{-0.1cm}\geq \Epsilon^{(1\leftarrow 2)}_{D(1),[\br^{(2)}]_{D(1)}}\hspace{-0.08cm}(\bx) \}
\end{align}
with $\Epsilon^{(1\leftarrow 2)}_{k,[\br^{(1)}]_k}$ the hyperplane with slope $A^{(1)^T}_{\br^{(1)}}A^{(2)}_{\br^{(2)}}$ and bias $\langle [A^{(2)}_{\br^{(2)}}]_{k,r,.},B^{(1)}_{\br^{(1)}}\rangle+B^{(2)}_{\br^{(2)}}$,$k\in \{1,\dots,D(1)\}$.
\end{lemma}
From Lemma~\ref{lemma:vertical_layer}, $\P^{(1\leftarrow 2)}_{\br^{(1)}}$ induces an underlying PD on its domain $\omega^{(1)}_{\br^{(1)}}$ that subdivides the cell into a PD denoted as $\PD^{(1\leftarrow 2)}_{\br^{(1)}}$ leading to the centroids
$\mu^{(1\leftarrow 2)}_{\br^{(1)},\br^{(2)}}={A^{(1)}_{\br^{(1)}}}^\top \mu^{(1\leftarrow 2)}_{\br^{(2)}},\label{eq:2layer_centroid}$ and radii $
 {\rm rad}^{(1\leftarrow 2)}_{\br^{(1)},\br^{(2)}}=\| \mu^{(1\leftarrow 2)}_{\br^{(1)},\br^{(2)}}\|^2 +2 \langle {\mu^{(2)}_{\br^{(2)}}}, B^{(1)}_{\br^{(1)}}\rangle+2\langle \textbf{1},B^{(2)}_{\br^{(2)}}\rangle, \forall \br^{(2)}\in\{1,\dots,R\}^{D(2)}$. The PD parameters thus combine the affine parameters $A^{(1)}_{\br^{(1)}},B^{(1)}_{\br^{(1)}}$ of the considered cell with the second layer parameters $A^{(2)},B^{(2)}$. Repeating this subdivision process for all cells $\br^{(1)}$ from $\Omega^{(1)}$ form the input space partitioning
$
    \Omega^{(1,2)}=\cup_{\br^{(1)}}\PD^{(1\leftarrow 2)}_{\br^{(1)}}.
$ 
\\
{\bf Recursion step}:
Consider the situation at layer $\ell$ knowing  $\Omega^{(1,\dots,\ell-1)}$ from the previous subdivision steps. following the intuition from the $\ell=2$, layer $\ell$ subdivides each cell in $\Omega^{(1,\dots,\ell-1)}$ to produce $\Omega^{(1,\dots,\ell)}$
leading to the -up to layer $\ell$-layer DN partitioning defined as 
\begin{align}
    \Omega^{(1,\dots,\ell)}=\cup_{\br^{(1)},\dots,\br^{(\ell-1)}}\PD^{(1 \leftarrow \ell)}_{\br^{(1)},\dots,\br^{(\ell-1)}} .\label{eq:omega}
\end{align}

\begin{thm}
\label{thm:subdivi2}
Each cell $\omega^{(1,\dots,\ell-1)}_{\br^{(1)},\dots,\br^{(\ell-1)}} \in \Omega^{(1,\dots,\ell-1)}$ is subdivided into $\PD^{(1\leftarrow \ell)}_{\br^{(1)},\dots,\br^{(\ell-1)}}$, a PD with domain $\omega^{(1.\dots.\ell-1)}_{\br^{(1)},\dots,\br^{(\ell-1)}}$ and parameters

\begin{align}
 \mu^{(1\leftarrow \ell)}_{\br^{(1)},\dots,\br^{(\ell)}}=&(A^{(1\leftarrow \ell-1)}_{\br^{(1)},\dots,\br^{(\ell-1)}})^\top \mu^{(\ell)}_{\br^{(\ell)}}&&\text{(centroids)}\\
 {\rm rad}^{(1\leftarrow \ell)}_{\br^{(1)},\dots,\br^{(\ell)}}=&-\| \mu^{(1\leftarrow \ell)}_{\br^{(1)},\dots,\br^{(\ell)}}\|^2 -2 \langle \mu^{(\ell)}_{\br^{(\ell)}}, B^{(1\rightarrow \ell-1)}_{\br^{(1)},\dots,\br^{(\ell-1)}}\rangle-2\langle \textbf{1},B^{(\ell)}_{\br^{(\ell)}}\rangle&&\text{(radii)},
 \label{eq:llayer_bias}
\end{align}
$\forall \br^{(i)}\in \{1,\dots,R\}^{D^{(i)}}$
with $B^{(1\rightarrow \ell-1)}=\sum_{\ell'=1}^{\ell-1}\big(\prod_{i=\ell-1}^{\ell'}A^{(i)}_{\br^{(i)}}\big)B^{(\ell')}_{\br^{(\ell')}}$
forming $\Omega^{(1,\dots,\ell)}$.
\end{thm}

The described recursion construction also provides a direct result on the shape of the entire DN input space partitioning cells.
\begin{cor}
The cells of the DN input space partitioning are convex polygons.
\end{cor}

\subsection{Combinatorial Geometry Properties}
\label{sec:cell_properties}

We highlight a key computational property of DNs contributing to their success.
While the actual number of cells from a layer PD varies greatly depending on the parameters, the cell inference task always search over the maximum $R^{(\ell)}_{{\rm upper}}=R^{D(\ell)}$ number of cells as
\begin{align}
\br^{(\ell)}(\bx) = \argmin_{\br \in \{1,\dots,R\}^{D(\ell)}}  \|\bz^{(\ell-1)}(\bx) - \mu^{(\ell)}_{\br} \|+{\rm rad}^{(\ell)}_{\br}.\label{eq:q_inference}
\end{align}
The computational and memory complexity of this task is $\mathcal{O}(R^{(\ell)}_{{\rm upper}}D(\ell-1))$. While approximations exist \cite{muja2009fast,arya1998optimal,georgescu2003mean}, we demonstrate how a MASO induced PD is constructed in such a way that it is parameter-,  memory-, and computation-efficient.

\begin{lemma}
A DN layer solves (\ref{eq:q_inference}) with computational and memory complexity $\mathcal{O}(\log_{R^{(\ell)}}(R^{(\ell)}_{{\rm upper}})R^{(\ell)}D(\ell-1))=\mathcal{O}(D(\ell)R^{(\ell)}D(\ell-1))$  as opposed to $\mathcal{O}(R^{(\ell)}_{{\rm upper}}D(\ell-1))=\mathcal{O}((R^{(\ell)})^{D(\ell)}R^{(\ell)}D(\ell-1))$ for an arbitrary Power Diagram. 
\end{lemma}

The entire DN then solves iteratively (\ref{eq:q_inference}) for each layer.

\begin{thm}
An entire DN infers an input cell with computational and memory complexity $\mathcal{O}(\sum_{\ell=1}^LD(\ell)R^{(\ell)}D(\ell-1))$  as opposed to $\mathcal{O}(\sum_{\ell=1}^L(R^{(\ell)})^{D(\ell)}R^{(\ell)}D(\ell-1))$ for an arbitrary hierarchy of Power Diagrams. 
\end{thm}

The above practical result is crucial to the ability of DN layers to perform extremely fine grained input space partitioning without sacrificing computation time especially at test time where one needs only feed forward computations of an input to obtain a prediction.

\subsection{Centroid and Radius Computation}
\label{sec:computation}
\begin{figure}
    \centering
    \hspace{0.27cm}Input $\bx$\hspace{0.32cm} 
    $\mu^{(1)}_{\bx}$\hspace{0.39cm}
    $\mu^{(1,2)}_{\bx}$\hspace{0.3cm}
    $\mu^{(1,2,3)}_{\bx}$\hspace{0.02cm}
    $\mu^{(1,\dots,4)}_{\bx}$\hspace{0.02cm}
    $\mu^{(1,\dots,5)}_{\bx}$\hspace{0.02cm}
    $\mu^{(1,\dots,6)}_{\bx}$\hspace{0.02cm}
    $\mu^{(1,\dots,7)}_{\bx}$\hspace{0.02cm}
    $\mu^{(1,\dots,8)}_{\bx}$\\
    \includegraphics[width=1\linewidth]{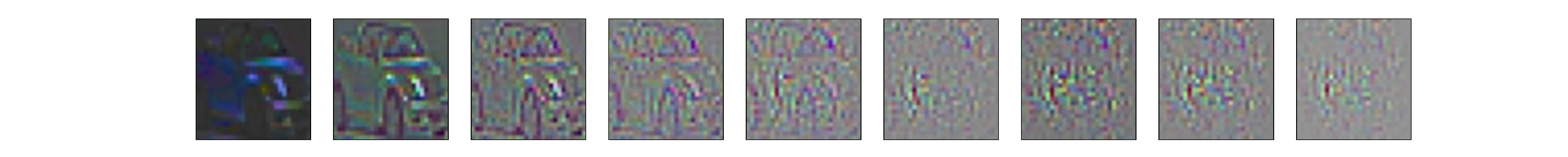}\\
    \includegraphics[width=1\linewidth]{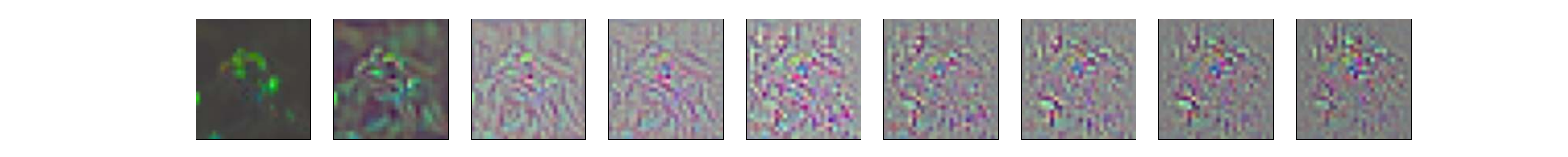}\\
    \includegraphics[width=1\linewidth]{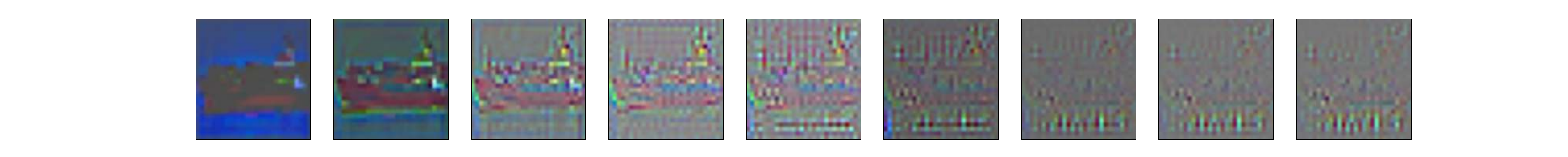}\\
    \includegraphics[width=1\linewidth]{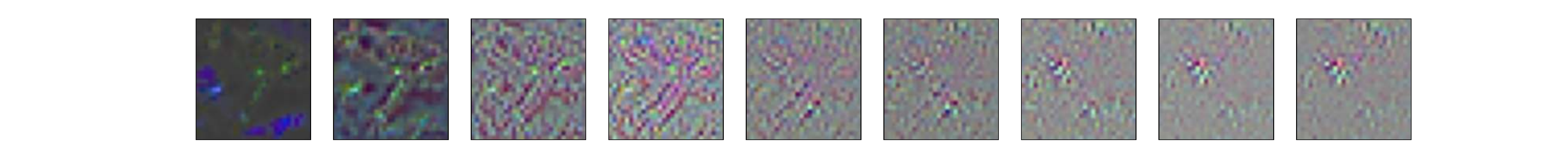}
    \caption{Depiction of the  centroids of the PD regions that contain an input $\bx$. This input belongs to a specific region $\omega^{(1,\dots,\ell)}_{\bx}$ for each successively refined PD subdivision of each layer $\Omega^{(1,\dots,\ell)}$, and in each case, the region has an associated centroid depicted here and radius. As depth increase, as the radii overtakes the centroids pushing $\mu^{(1,\dots,\ell)}_{\bx}$ away from the region.}
    \label{fig:centroids}
\end{figure}

In practice, the number of centroids and radius for each of the partitioning $\Omega^{(1,\dots, \ell)}$ contains too many cells to compute all the centroids and radius. However, given a cell (resp. a point $\bx$) and an up to layer $\ell$ code $\br^{(1)},\dots,\br^{(\ell)}$ (resp. $\br^{(1)}(\bx),\dots,\br^{(\ell)}(\bx)$), computing the centroid and radius can be done as follows:
\begin{align}
A^{(1\rightarrow \ell-1)}_{\bx}=&({\rm J}_{\bx}f^{(1 \rightarrow \ell-1)})=\Big(\nabla_{\bx}f^{(1 \rightarrow \ell-1)}_1,\dots,\nabla_{\bx}f^{(1 \rightarrow \ell-1)}_{D(\ell)}\Big)^\top\\
\mu^{(1\leftarrow \ell)}_{\bx} =&{A^{(1\rightarrow \ell-1)}_{\bx}}^\top \sum_{k=1}^{D(\ell)}[A^{(\ell)}_{\bx}]_{k,.}=\sum_{k=1}^{D(\ell)}\nabla_{\bx}f^{(1\rightarrow \ell)}_k,\label{eq:all_centroid}\\
{\rm rad}^{(1\leftarrow \ell)}_{\bx}=&-\| \mu^{(1\leftarrow \ell)}_{\bx}\|^2 -2\langle \textbf{1},B^{(\ell)}_{\bx}\rangle
-2 \langle f^{(1 \rightarrow \ell)}(\bx)-A^{(1\rightarrow \ell-1)}_{\bx}\bx,\sum_{k=1}^{D(\ell)}[A^{(\ell)}_{\bx}]_{k,.}\rangle
\end{align}
where we remind that $\mu^{(\ell)}_{\bx}=\sum_{k=1}^{D(\ell)}[A^{(\ell)}_{\bx}]_{k,.}$ and $B^{(1\rightarrow \ell-1)}_{\bx}=f^{(1 \rightarrow \ell)}(\bx)-A^{(1\rightarrow \ell-1)}_{\bx}\bx$ from Theorem~\ref{thm:subdivi2}. Notice how centroids and biases of the current layer are mapped back to the input space $\mathcal{X}^{(0)}$ via projection onto the tangent hyperplane with basis given by $A^{(1\rightarrow \ell-1)}_{\bx}$. 
\begin{prop}
The centroids correspond to the backward pass of DNs and thus can be computed efficiently by backpropagations.
\end{prop}
Note how the form in (\ref{eq:all_centroid}) correspond to saliency maps. In particular, at a given layer, the centroid of the region in which $\bx$ belongs is obtained by summing all the per unit saliency maps synonym of adding all the unit contributions in the input space.
We provide in Fig.~\ref{fig:centroids} computed centroids for a trained Resnet on CIFAR10, for each PD subdivision, see appendix~\ref{appendix_resnet} for details on the model, performance and additional figures. The ability to retrieve saliency maps and the form of the centroid opens the door to further use in many settings of the centroids. For example.  semi supervised learning successfully leveraged the last layer centroid in \cite{balestriero2018semi} by providing a loss upon them.

\subsection{Empirical Region Characterization}

\begin{figure}[t!]
\begin{minipage}{0.05\linewidth}
    \rotatebox{90}{\hspace{0.6cm}MLP \hspace{0.65cm}CNN\hspace{1.9cm}MLP \hspace{0.65cm}CNN\hspace{1cm}}
\end{minipage}
\begin{minipage}{0.6\linewidth}
    \centering
    CIFAR10\\
    \includegraphics[width=0.48\linewidth]{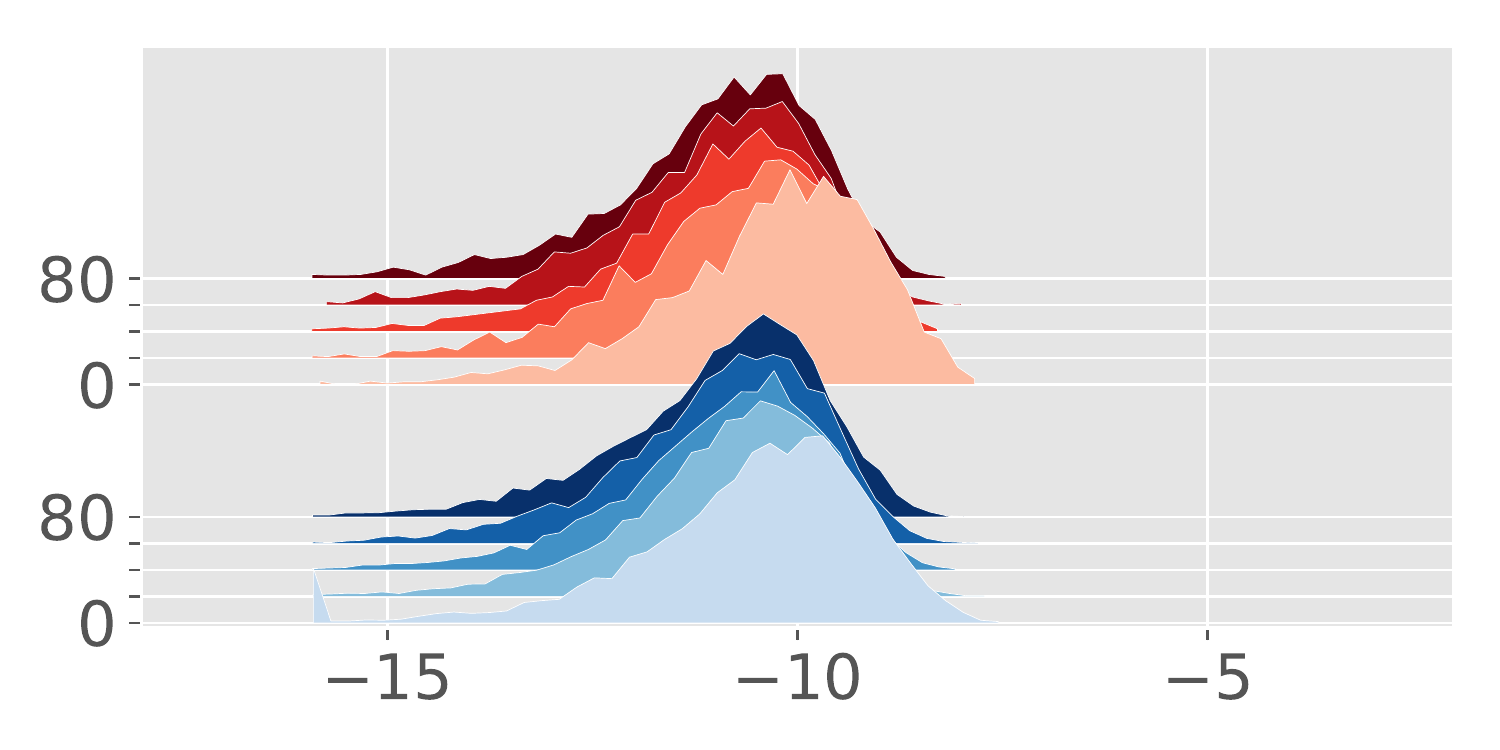}
    \includegraphics[width=0.48\linewidth]{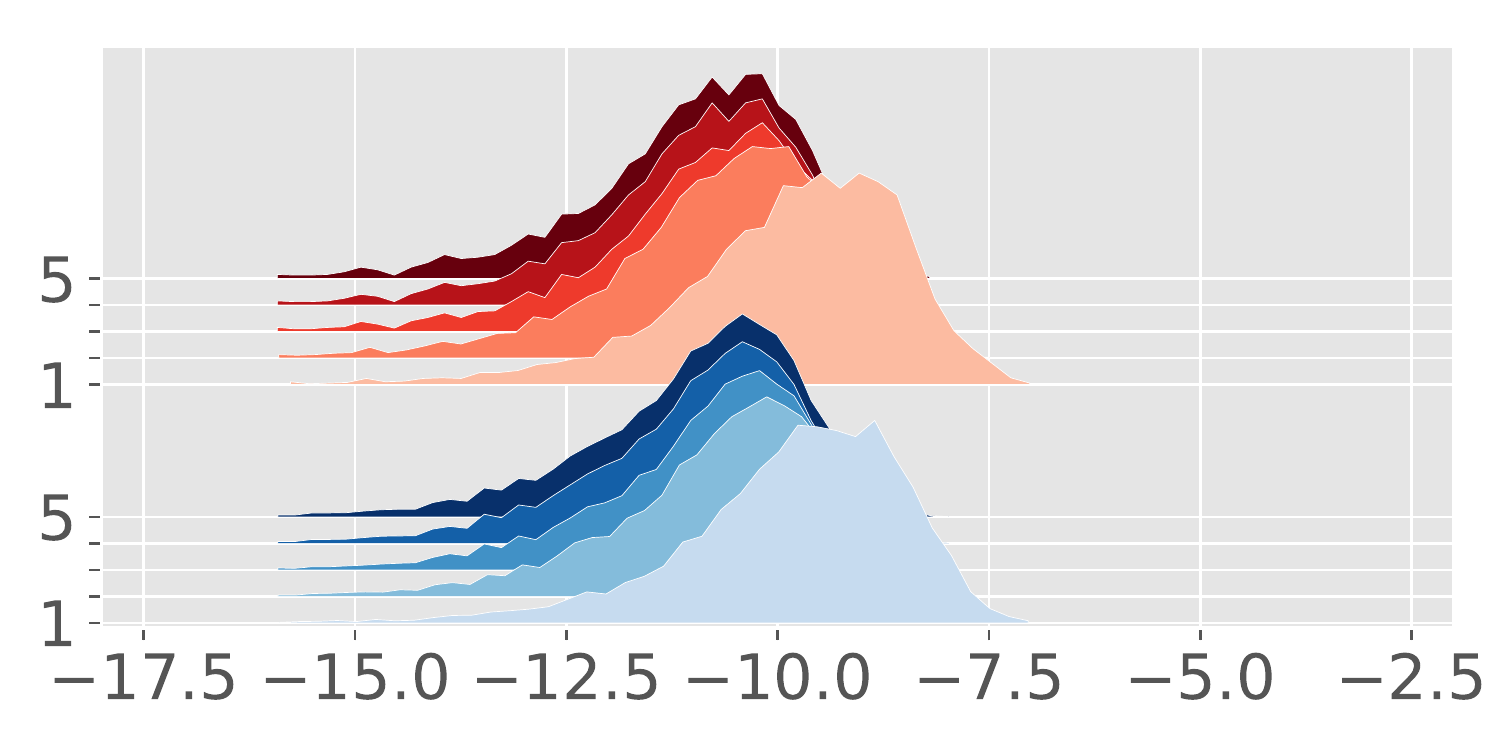}\\
    \includegraphics[width=0.48\linewidth]{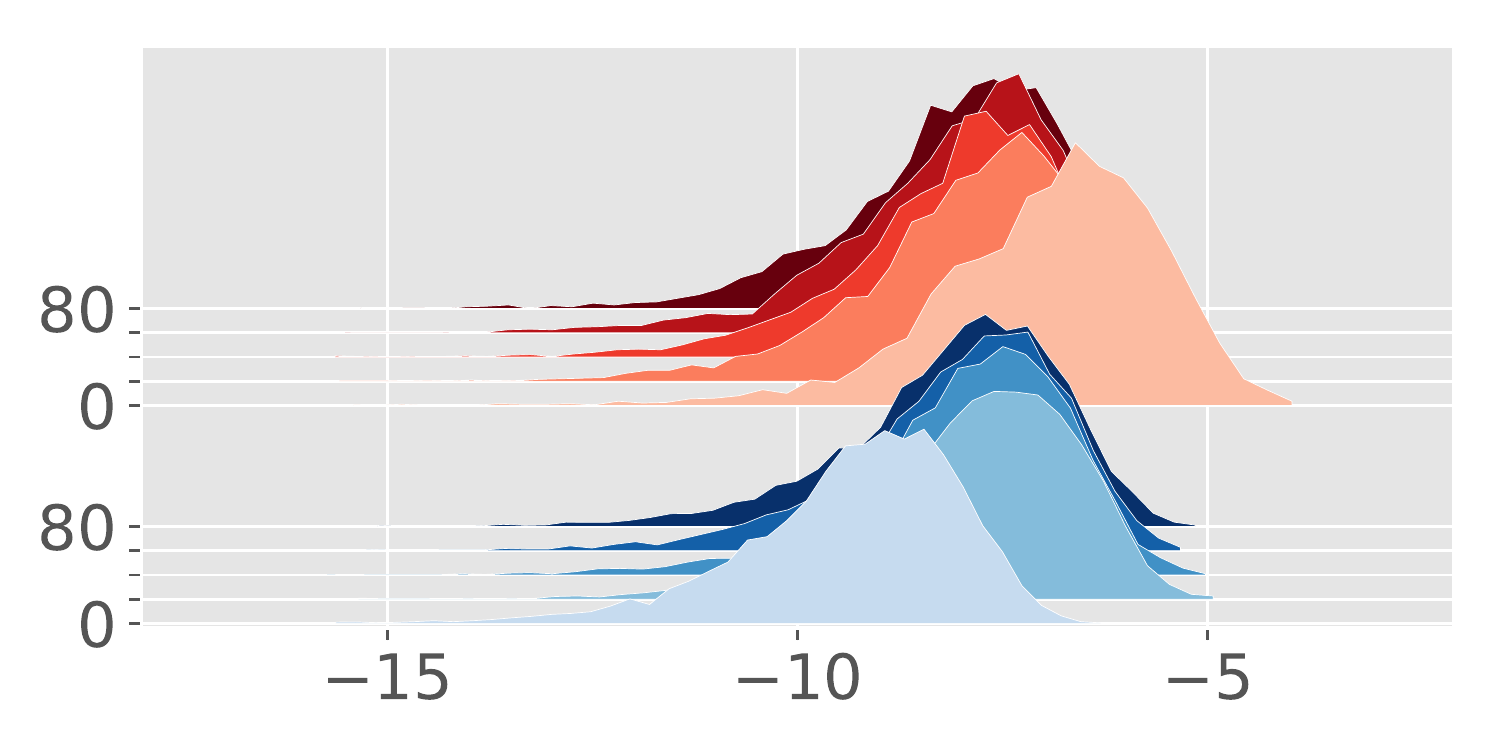}
    \includegraphics[width=0.48\linewidth]{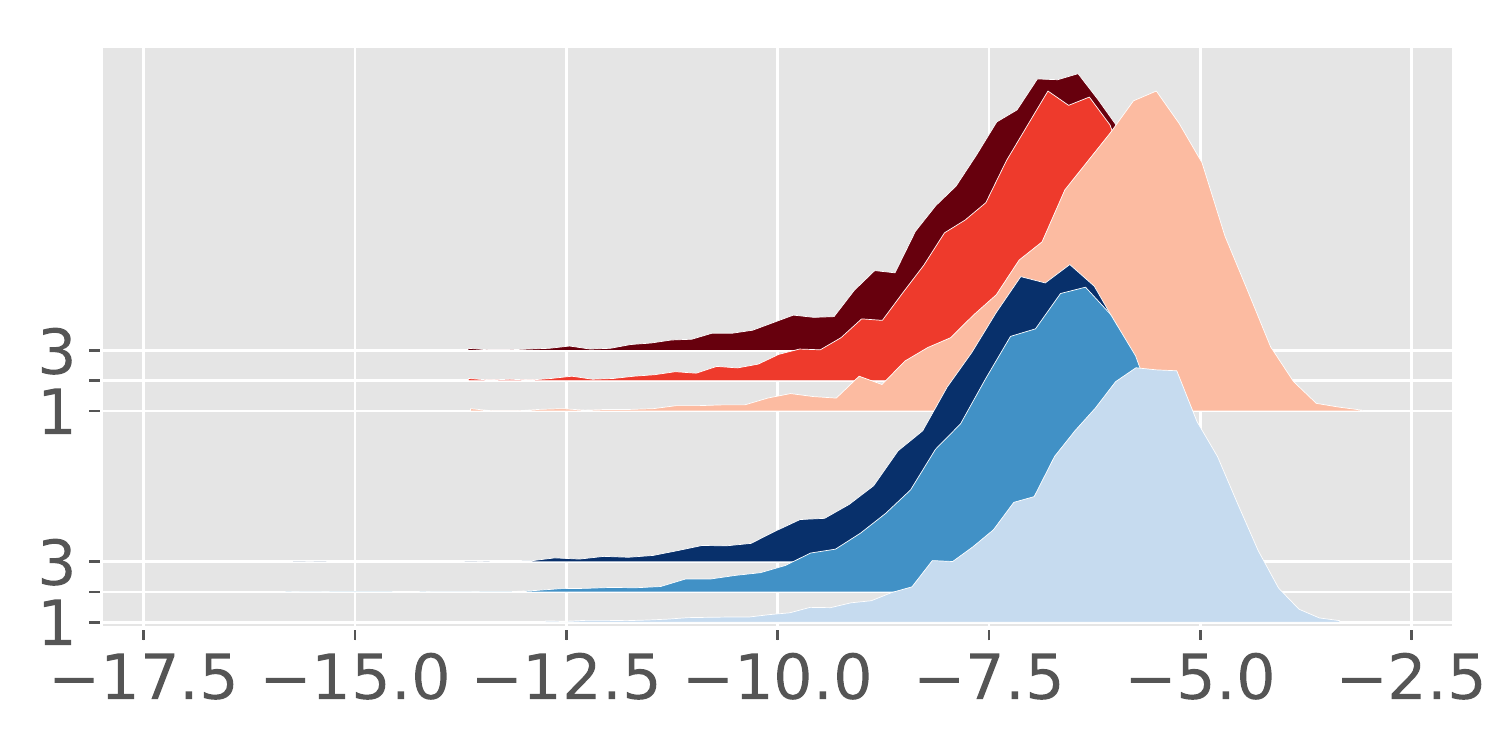}\\
    SVHN\\
    \includegraphics[width=0.48\linewidth]{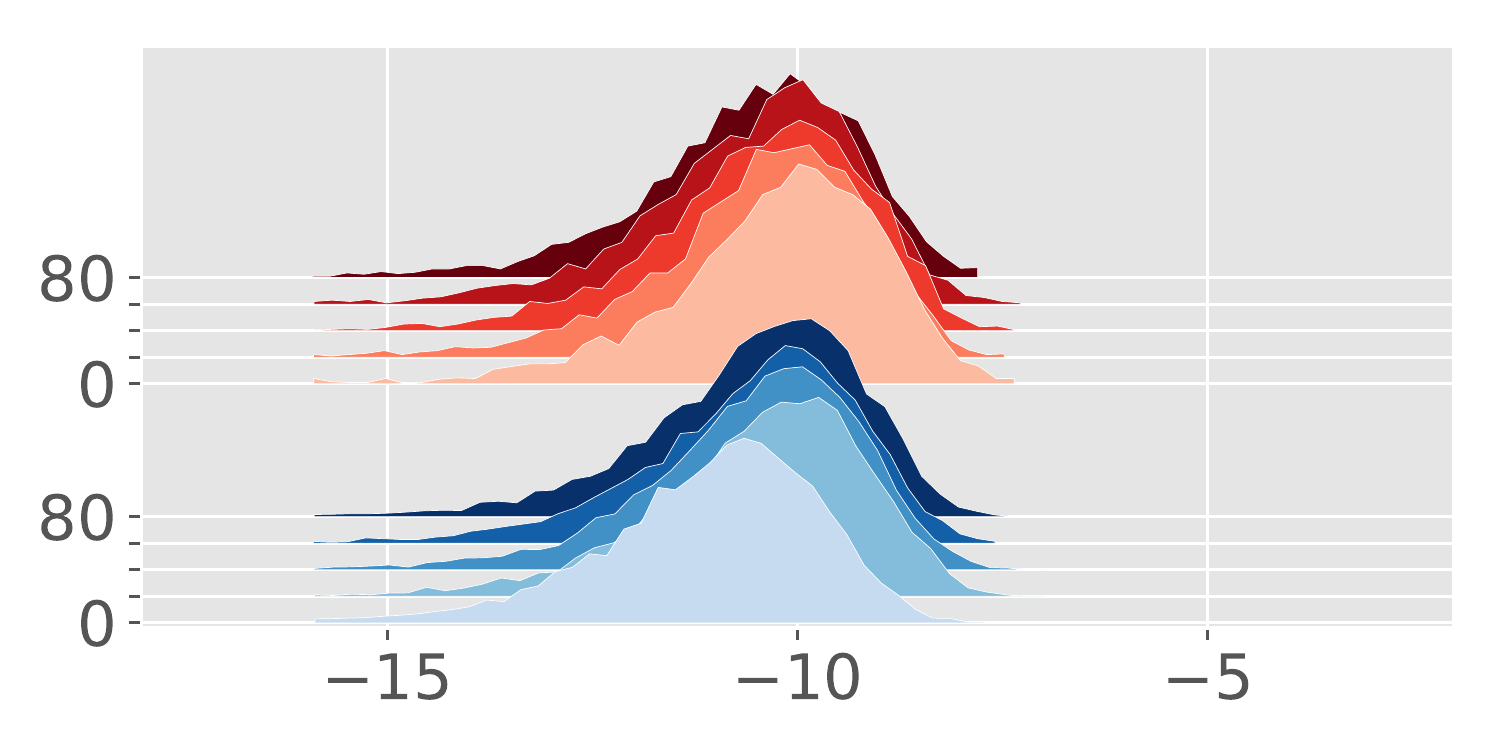}
    \includegraphics[width=0.48\linewidth]{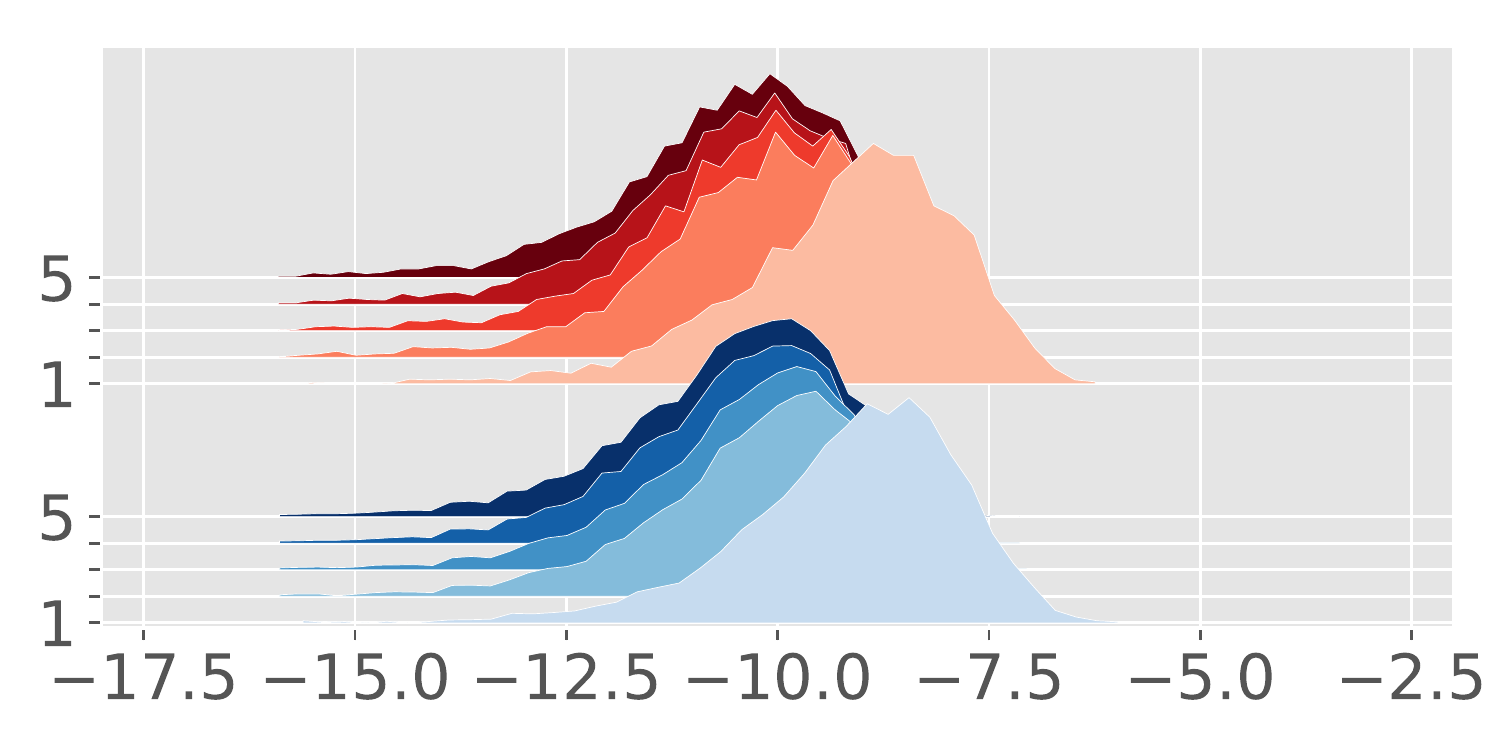}\\
    \includegraphics[width=0.48\linewidth]{DISTANCES/loghistogram_epochs_save_dense_test_v2_cifar10_False_l2.pdf}
    \includegraphics[width=0.48\linewidth]{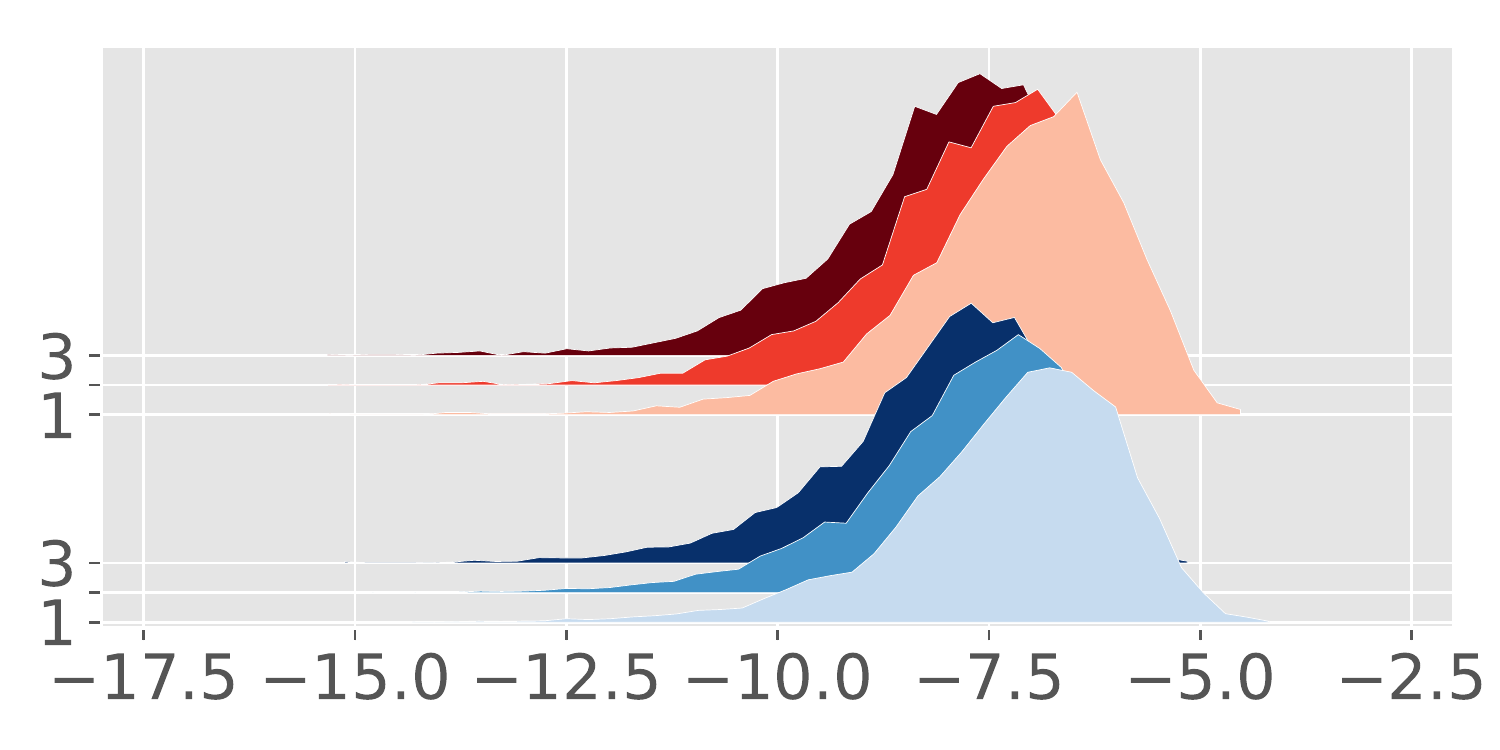}
\end{minipage}
\begin{minipage}{0.34\linewidth}
    \caption{\small Depiction of the distribution of the log of the distances through the epochs (during leaning) for the left column and through the layer (as the partitioning gets subdivided) on the right column. The statistics are computed on the train set (blue) and test set (red). CLear insight in the role of the deeper layer trying to refine only some regions, likely the ones hard to classify. This is shown by the tails becoming heavier. For additional figures see Fig.~\ref{fig:additional_detail3}.\normalsize}
    \label{fig:distances_distribution}
\end{minipage}
\end{figure}
We provide in Fig.~\ref{fig:distances_distribution} the distribution of distances from the dataset points to the nearest region boundaries of the input space partitioning for each layer (at a current subdivision step) and at different stages of the training procedure. Clearly, training slightly impacts those distances and slight increase the number of inputs with small distance to then nearest boundary. Yet, the main impact of learning resides in shaping the regions via learning of the weights. We also train a DN on various dataset and study how many inputs share the same input space partitioning region. We observed that from initialization to the end of the learning, there are never more than one image in a given region and this with standard DNs providing near or state of the art performances. Yet, drastically reducing the size of the DN allows to have more than one image per regions when considering the first steps of the subdivision.
\section{Geometry of a Deep Network Decision Boundary}
\label{sec:boundary}

The analysis of the input space partitioning was achieved by successively expressing the layer $\ell$ polytope in the input space. While Sec.~\ref{sec:DEEP} focused on the faces of the polytope which define the cells in the input space, we now turn to the edges of the polytope which define the cells' boundaries. In particular, we demonstrate how a single unit at layer $\ell$ defines multiple cell boundaries in the input space, and use this finding to finally derive the analytical formula of the DN decision boundary in classification tasks. In this section we focus on DN nonlinearities using $R=2$ nonlinearities such as ReLU,  leaky-ReLU, and absolute value.

\subsection{Partitioning Boundaries and Edges}

In the case of $R=2$ nonlinearities, the polytope $\mathcal{P}^{(\ell)}_k$ of unit $z^{(\ell)}_k$ contains a single edge. This edge can be expressed in some space $\mathcal{X}^{(\ell')},\ell'<\ell$, as a collection of continuous piecewise linear paths.

\begin{defn}
\label{def:edge}
The edges of the polytope $\mathcal{P}^{(\ell)}_k$ in some space $\mathcal{X}^{(\ell')},\ell'<\ell$  are the collection of points defined as
\begin{align}
    {\rm edge}_{\mathcal{X}^{(\ell')}}(k,\ell)  =\{\bx \in \mathcal{X}^{(\ell')} : \Epsilon^{(\ell-1)}_{k,2}(\bz^{(\ell'\rightarrow \ell)}(\bx))=0 \},\label{eq:edge}
\end{align}
with $\bz^{(\ell'\rightarrow \ell-1)}=\bz^{(\ell-1)} \circ \dots \circ \bz^{(\ell')}$.
\end{defn}
Thus the edges correspond to the level curve of the unit in $\mathcal{X}^{(\ell')}$.
Defining the polynomial
\begin{align}
    {\rm Pol}^{(\ell)}(\bx)=\prod_{k=1}^{D(\ell')}(z_k^{(\ell)} \circ \bz^{(\ell-1)}\circ \dots \circ \bz^{(1)})(\bx),\label{eq:polynomial}
\end{align}
we obtain the following result where the boundaries of $\Omega^{(1,\dots,\ell)}$ from Theorem~\ref{thm:subdivi2} can be expressed in term of the polytope edges and roots of the polynomial.

\begin{thm}
The polynomial (\ref{eq:polynomial}) is of order $\prod_{\ell=1}^LD(\ell)$, its roots correspond to the partitioning boundaries:
\begin{align}
    \partial \Omega^{(1,\dots,\ell)}=\cup_{\ell'=1}^{\ell} \cup_{k=1}^{D(\ell')} {\rm edge}_{\mathcal{X}^{(0)}}(k,\ell)=\{\bx \in \mathcal{X}^{(0)}:\prod_{\ell=1}^{L}{\rm Pol}^{(\ell)}(\bx)=0\}
\end{align}
and the root order defines the dimension of the root (boundary, corner, ...).
\end{thm}

\subsection{Decision Boundary and Curvature}
In the case of classification, the last layer typically includes a softmax nonlinearity and is thus not a MASO layer. However, \cite{balestriero2018from} demonstrated that it can be expressed as a MASO layer without any change in the model and output. As a result, this layer introduces a last subdivision of the DN partitioning. We focus on binary classification for simplicity of notations, in this case, $D(L)=1$ and a single last subdivision occurs. In particular, using the previous result we obtain the following.

\begin{prop}The decision boundary of a DN with $L$ layers is the edge of the last layer polytope $\P^{(L)}$ expressed in the input space $\mathcal{X}^{(0)}$ from Def.~\ref{def:edge} as
\begin{align}
{\rm Decision Boundary}=\{\bx \in \mathcal{X}^{(0)}:f(\bx)=0\}={\rm edge}_{\mathcal{X}^{(0)}}(1,L)    \subset \partial \Omega^{(1,\dots,L)}.
\end{align}
\end{prop}

To provide insights let consider a $3$ layer DN denoted as $f$ and the binary classification task; we have as the DN induced decision  boundary the following
\begin{align}
{\rm Decision Boundary} &=\cup_{\br^{(2)}}\cup_{\br^{(1)}} \{x \in \mathcal{X}^{(0)}: \langle \alpha_{\br^{(2)},\br^{(1)}},\bx\rangle +\beta_{\br^{(2)},\br^{(1)}}=0\}\cap \omega^{(1,2)}_{\br^{(1)},\br^{(2)}},\label{eq:decision}
\end{align}
with $\alpha_{\br^{(1)},\br^{(2)}}=(A^{(2)}_{\br^{(2)}}A^{(1)}_{\br^{(1)}})^T[A^{(3)}]_{1,1,.}$ and $\beta_{\br^{(1)},\br^{(2)}}=[A^{(3)}]_{1,1,.}^TA^{(2)}_{\br^{(2)}}B^{(1)}_{\br^{(1)}}+[B^{(3)}]_{1,1}$.
Studying the distribution of $\alpha_{\br^{(1)},\br^{(2)}}$ characterizes the structure of the decision boundary and thus open the highlight the interplay between layer parameters, layer topology, and the decision boundary. For example, looking at Figure~\ref{fig:megaplot} and the red line demonstrates how the weight characterize the curvature and cuts position of the decision boundary.

We provide a direct application of the above finding by providing a curvature characterization of the decision boundary. First, we propose the following result stating that the form of $\alpha$ and $\beta$ from (\ref{eq:decision}) from one region to an neighbouring one only alters a single unit code at a given layer.
\begin{lemma}
Any edge as defined in Def.~\ref{def:edge} reaching a region boundary, must continue in this neighbouring region.
\end{lemma}

This comes directly from continuity of the involved operator. This demonstrates that the decision boundary as defined in (\ref{eq:decision}) can have its curvature defined by comparing the form of the edges of adjacent regions.
\begin{thm}
The decision boundary curvature/angle between two adjacent regions $\br$ and $\br'$\footnote{For clarity, we omit the subscripts.} is given by the following dihedral angle \cite{kern1938solid} between the adjacent hyperplanes as
\begin{equation}
    \cos (\theta(\br,\br')) = \frac{|\langle \alpha_{\br},\alpha_{\br'} \rangle|}{\| \alpha_{\br} \| \| \alpha_{\br'} \|}.\label{eq:distance}
\end{equation}
\end{thm}
\begin{wrapfigure}[6]{r}[4pt]{0.16\textwidth}
  \begin{center}
    \includegraphics[width=0.15\textwidth]{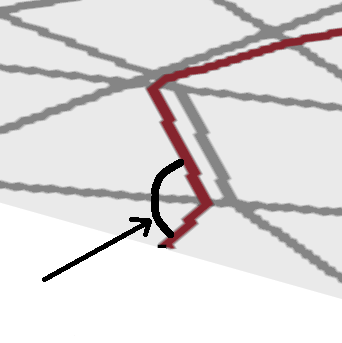}
  \end{center}
\end{wrapfigure}
The above is illustrated in the adjacent figure with one of the angle highlighted with an arrow.
The hyperplane offsets $\beta_{\br},\beta_{\br'}$ are irrelevant to the boundary curvature. Following this, the DN bias units are also irrelevant to the boundary curvature. Also, the norm of the gradient of the DN and thus the Lipschitz constant alone does not characterize the regularity of the decision boundary. In fact, the angle is invariant under scaling of the parameters. This indicates how measures based on the input-output sensitivity do not characterize alone the curvature of the decision boundary.

Finally, we highlight an intuitive result that can be derived from the above.
We remind that a neighbouring regions implies a change of a single code unit. Let denote without loss of generality the changed code index by $d'\in \{1,\dots,D^{(1)}$. The other $D^{(1)}-1$ codes remain the same.
When dealing with $R=2$ nonlinearities, this implies that $[\br]_{d'}$ changes from a $1$ to a $2$ for those two neighbouring regions. Let denote by $\br$ the case with a $1$ and by $\br'$ the case with a $2$. With those notations, we can derive some special cases of the distance formula (\ref{eq:distance}) for some DN topologies.

\begin{prop}
In a 2-layer DN with ReLU and  orthogonal first layer weights, we have
\begin{align}
\cos (\theta(\br,\br'))=
\left(
\frac{|[W^{(2)}]_{1,d'}| \| [W^{1)}]_{d',.} \|_2}{\sum_{d\not = d'}^{D^{(1)}}1_{\{[\br']_d=2\}}|[W^{(2)}]_{1,d}|\| [W^{(1)}]_{d,.}\|}+1
\right)^{-1}
\in (0,1)
\label{eq:2layerdistance}
\end{align}
\end{prop}

From the above formula it is clear that reducing the norm of the weights alone does not impact the angles. However, we have the following result.

\begin{prop}
Regions of the input space in which the amount of firing ReLU is small will have greater decision boundary curvature than regions with most ReLU firing simultaneously.
\end{prop}

As a result, a ReLU based DN will have different behavior at different regions of the space. And the angle is always positive. Interestingly, the use of absolute value on the other hand leads to the following.

\begin{prop}
In a 2-layer DN with absolute value and  orthogonal first layer weights, we have
\begin{align}
\cos (\theta(\br,\br'))
=1- 2 \left(
1+\sum_{d\not = d'}^{D^{(1)}}|[W^{(2)}]_{1,d}|^2\| [W^{(1)}]_{d,.}\|^2
\right)^{-1}
\in (0,1).
\label{eq:2layerdistanceabs}
\end{align}
\end{prop}

Now, not only are the angles  between $90$ degrees and $270$ (as opposed to ReLU between $90$ and $180$), but the angles also do not depend on the state of the other absolute values but just on the norm of the weights of both layers.



\section{Conclusions}
\label{sec:conc}

We have extended the understanding of DNs by leveraging computational geometry to characterize how a DN partitions its input space via a multiscale Power Diagram subdivision. 
Our analytical formulation for the partitions induced by not only the entire DN but also each of its units and layers will open new avenues to study how to optimize DNs for certain tasks and classes of signals.   


\small
\bibliographystyle{plainnat}

\bibliography{BIBLIO}

\appendix
\section{Additional Geometric Insights}
\label{appendix:extra_geometric}

The Laguerre distance corresponds to the length of the line segment that starts at $\bx$ and ends at the tangent to the hypersphere with center $[A^{(\ell)}]_{k,r^{\star},\bigcdot}$ and radius $\big[{\rm rad} \big]_{k,r^{\star}}$ (see Figure~\ref{fig:power_diagram}).

The hyperplanar boundary between two adjacent PD regions can be characterized in terms of the {\em chordale} of the corresponding hyperspheres \cite{johnson1960advanced}. Doing so for all adjacent boundaries fully characterize those region boundaries in simple terms of hyperplane intersections from \cite{aurenhammer1987power}.

\subsection{Paraboloid $U$ Insights}
\label{appendix:paraboloid}

A further characterization of the polytope boundary $\partial \P_k$ can be made by introducing the paraboloid $U$ defined as $U(\bx)= \frac{1}{2} \big\| \bx \big\|^2_2$. 
Notice that the slope of the hyperplane is $\nabla \Epsilon_{k,r}=[A]_{k,r,\bigcdot}$ and its offset is $-\frac{1}{2}\|[A]_{k,r,\bigcdot} \|^2_2$. Defining the paraboloid $U$ defined as $U(\bx)= \frac{1}{2} \big\| \bx \big\|^2_2$, we see how \textit{the hyperplane $\Epsilon
_{k,r}$ is the tangent of the paraboloid $U$ at the point $[A]_{k,r,\bigcdot}$}.
We now highlight that the hyperplane and the paraboloid intersect at an unique point
\begin{equation}
    U(\bx) \begin{cases} = \Epsilon_{k,r}(\bx)  \iff \bx= [A]_{k,r,\bigcdot}\\
    > \Epsilon_{k,r}(\bx)\text{ , else}
    \end{cases}
    \label{eq:U_intersection}
\end{equation}
The faces of $\P_k$ are the tangent of $U$ at the points given by $[A]_{k,r,\bigcdot},\forall r$ leading to
\begin{equation}
    U(\bx) \begin{cases} = \P_k(\bx)  \iff \bx \in \{ [A]_{k,r,\bigcdot}, \forall r\}\\
    > \P_k(\bx)\text{ , otherwise}
    \end{cases}
\end{equation}

Concerning the case of abitrary bias we hve the following insights.
We can characterize the hypersphere as being the intersection of the hyperplanes with the paraboloid in the following result from \cite{aurenhammer1987power}.

\begin{prop}
\cite{aurenhammer1987power}
There is a bijective mapping between the hyperpshere in the input domain and the intersection of the hyperplane $\EPSILON$ in $\mathbb{R}^{D+1}$ with the paraboloid $U$.
\end{prop}
In fact, the projection of the intersection between the hyperplane and the paraboloid onto the input space is forming a circle where the radius corresponds to the shift of the hyperplane.

\newpage
\section{Additional Figures}
\label{appendix:additional_figure}

\begin{figure}[ht]
    \centering
    \begin{minipage}{0.03\linewidth}
    \rotatebox{90}{\hspace{0.5cm}Layer 3\hspace{1.3cm}Layer 2\hspace{1.3cm} Layer 1}
    \end{minipage}
    \begin{minipage}{0.95\linewidth}
    \begin{minipage}{0.16\linewidth}
    \centering
    Unit 1\\
    \includegraphics[width=1\linewidth]{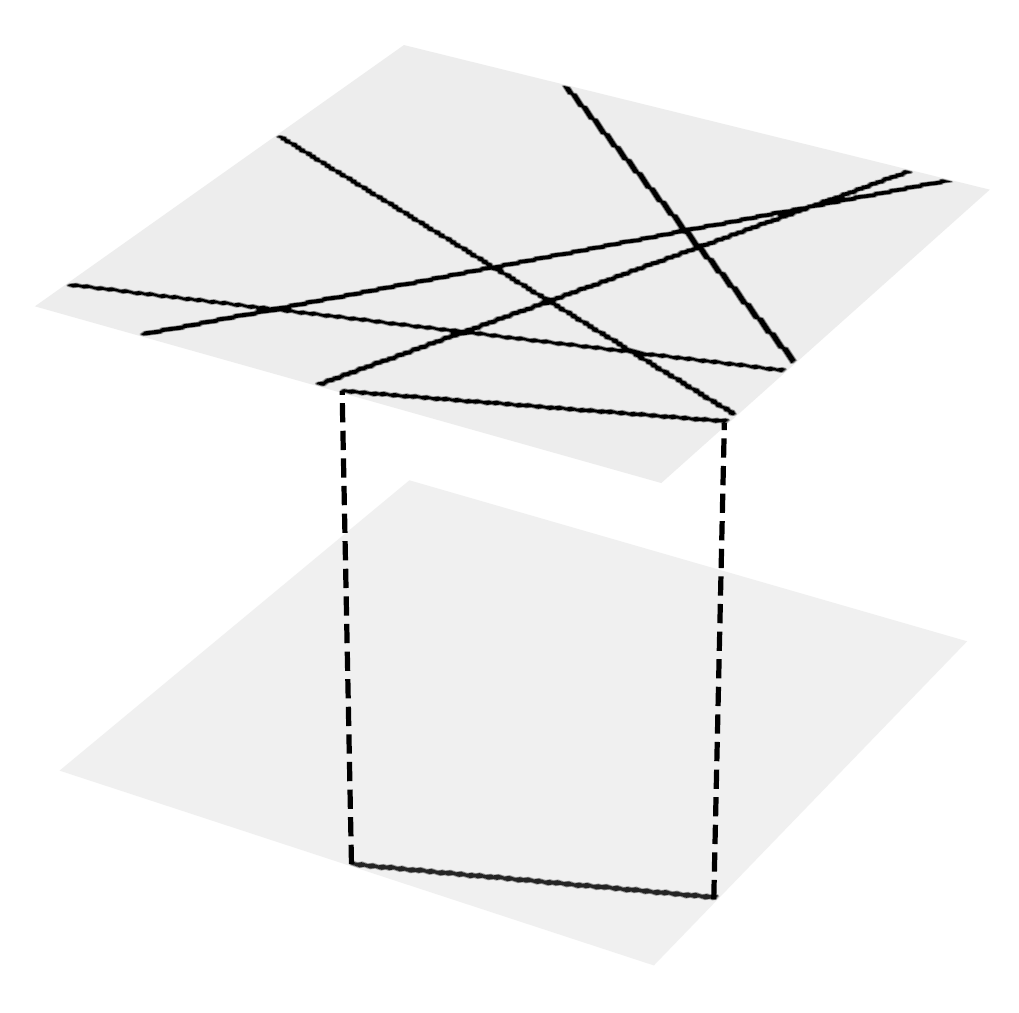}
    \end{minipage}\hfill
    \begin{minipage}{0.16\linewidth}
    \centering
    Unit 2\\
    \includegraphics[width=1\linewidth]{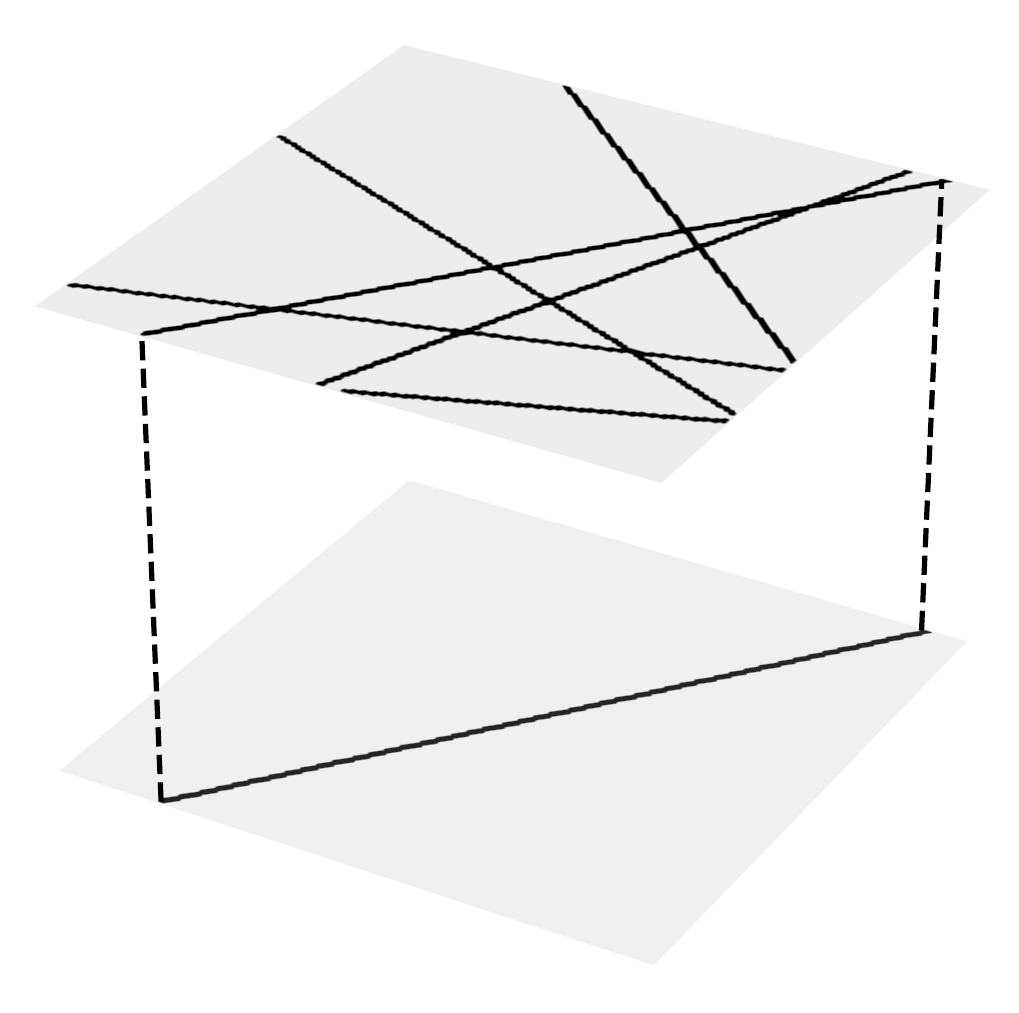}
    \end{minipage}\hfill
    \begin{minipage}{0.16\linewidth}
    \centering
    Unit 3\\
    \includegraphics[width=1\linewidth]{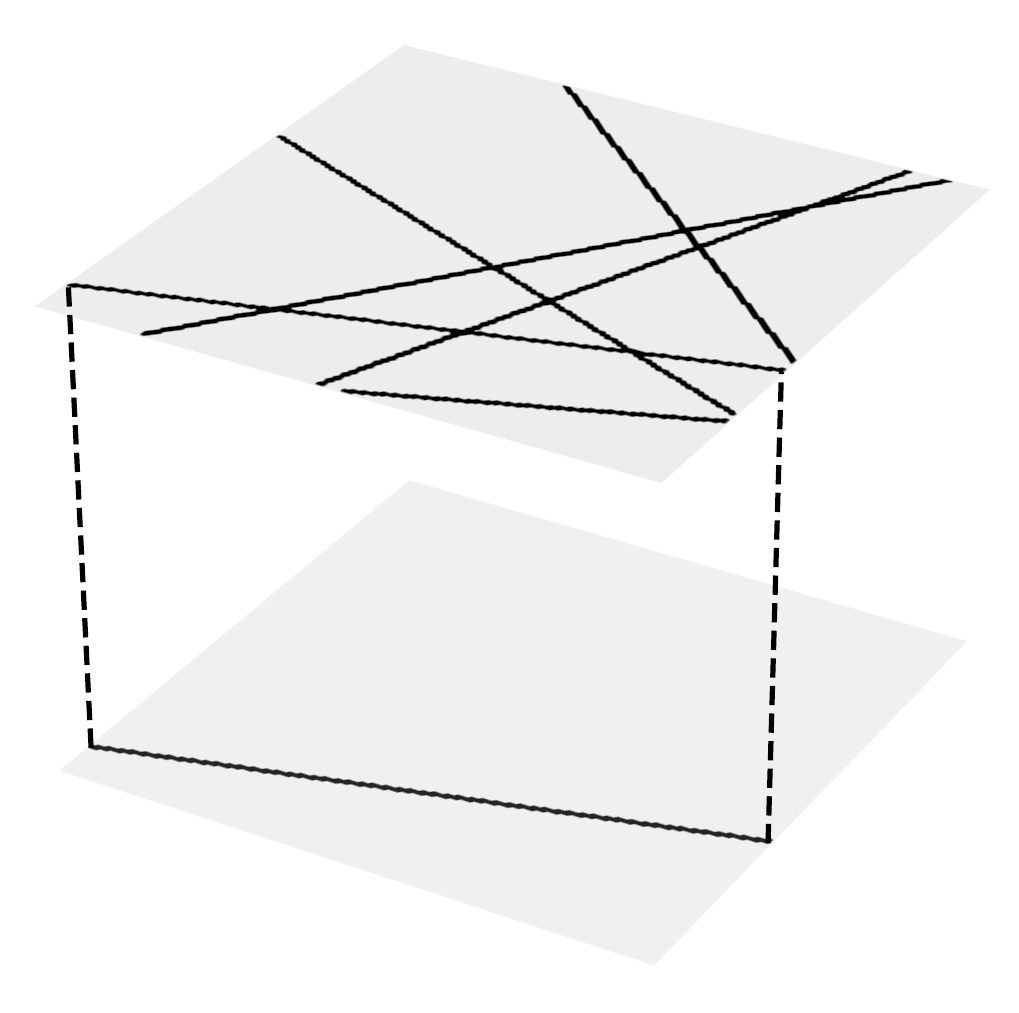}
    \end{minipage}\hfill
    \begin{minipage}{0.16\linewidth}
    \centering
    Unit 4\\
    \includegraphics[width=1\linewidth]{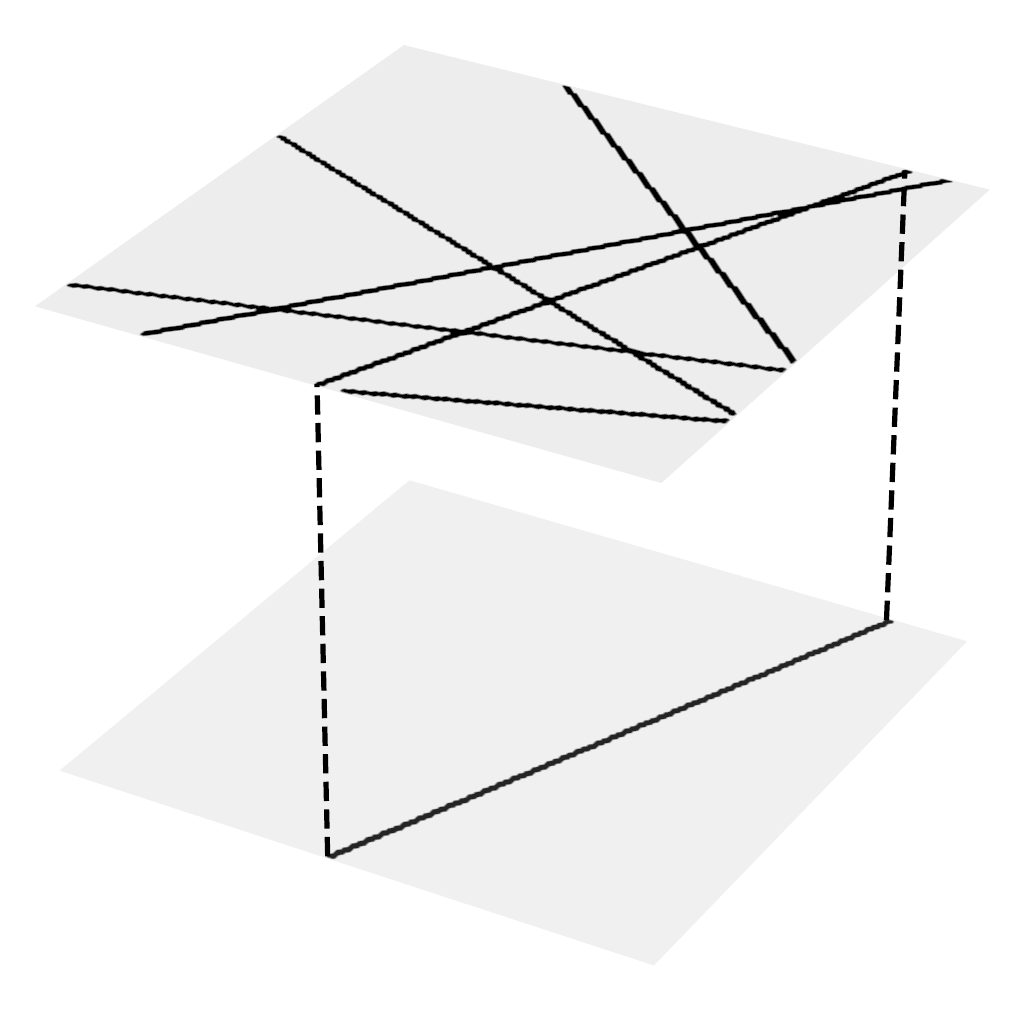}
    \end{minipage}\hfill
    \begin{minipage}{0.16\linewidth}
    \centering
    Unit 5\\
    \includegraphics[width=1\linewidth]{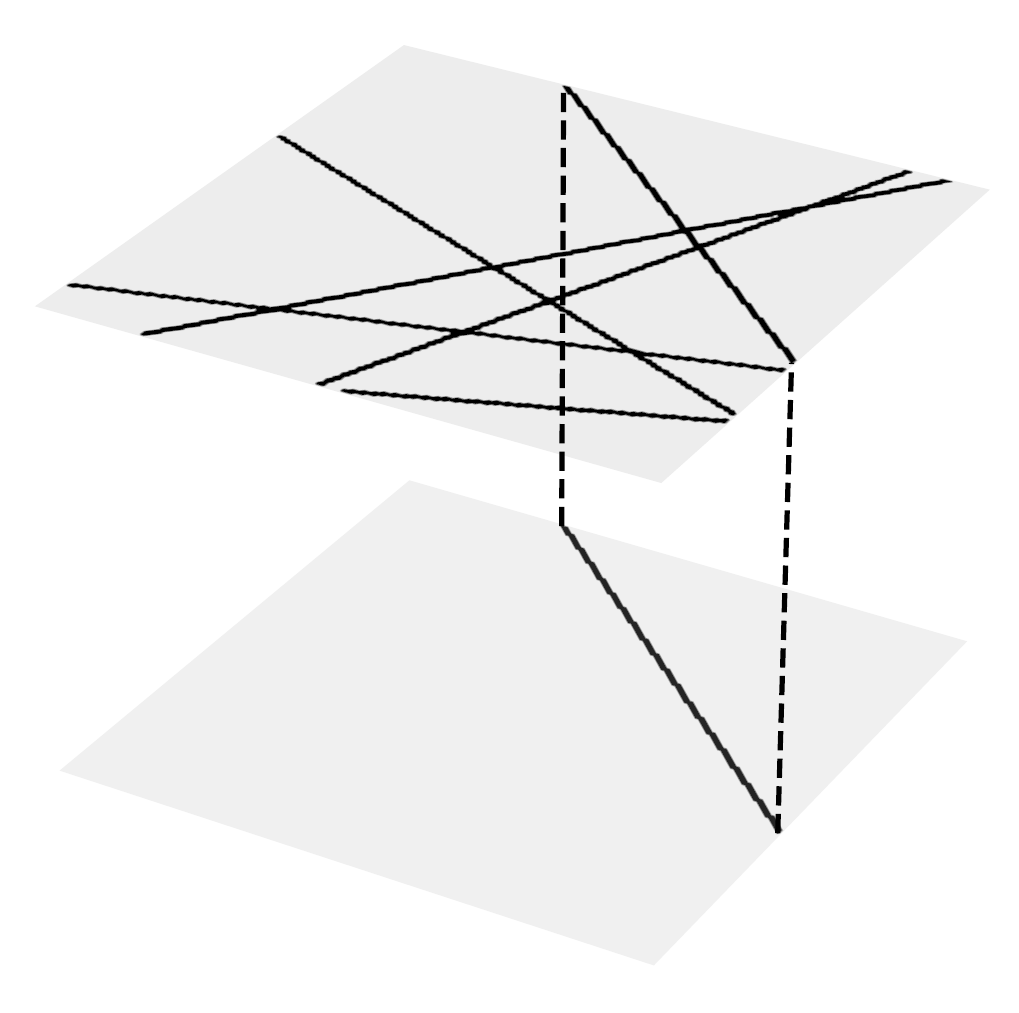}
    \end{minipage}\hfill
    \begin{minipage}{0.16\linewidth}
    \centering
    Unit 6\\
    \includegraphics[width=1\linewidth]{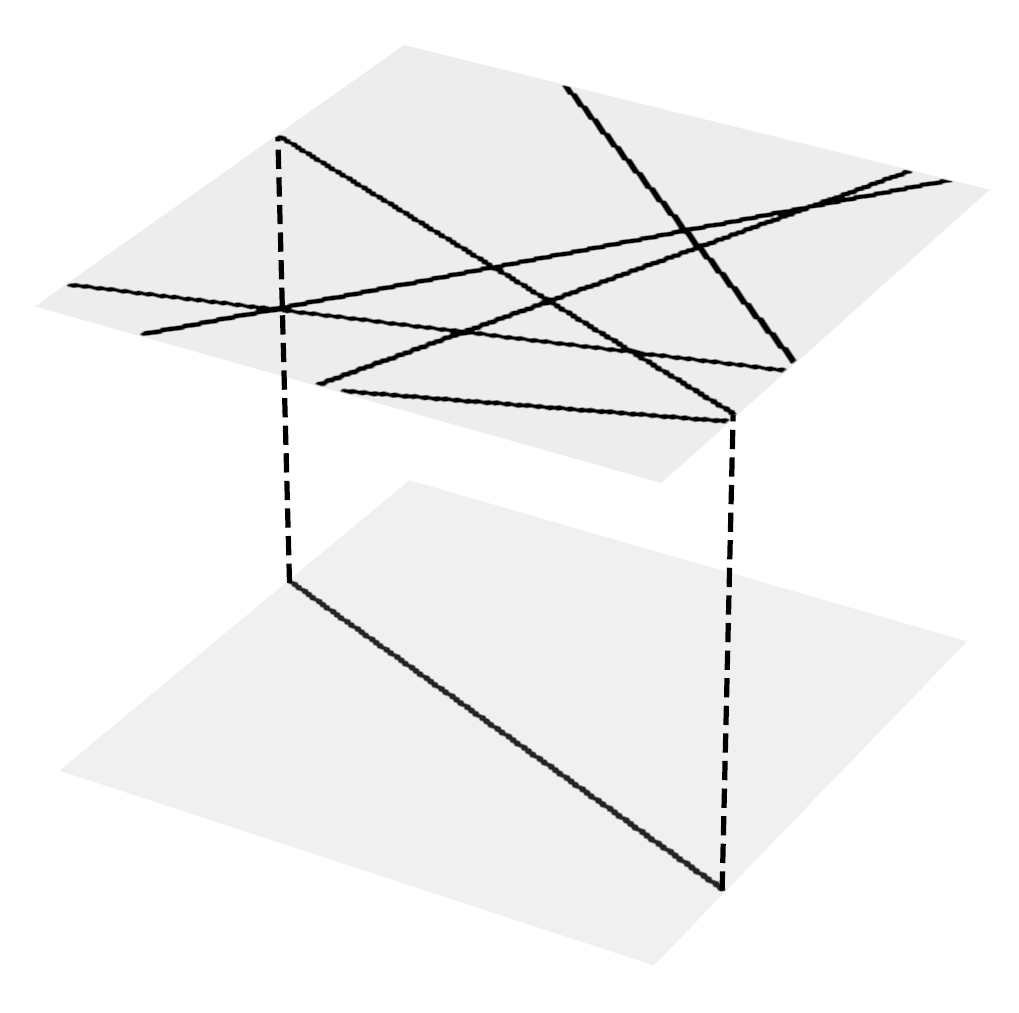}
    \end{minipage}\hfill\\
    \begin{minipage}{0.16\linewidth}
    \centering
    Unit 1\\
    \includegraphics[width=1\linewidth]{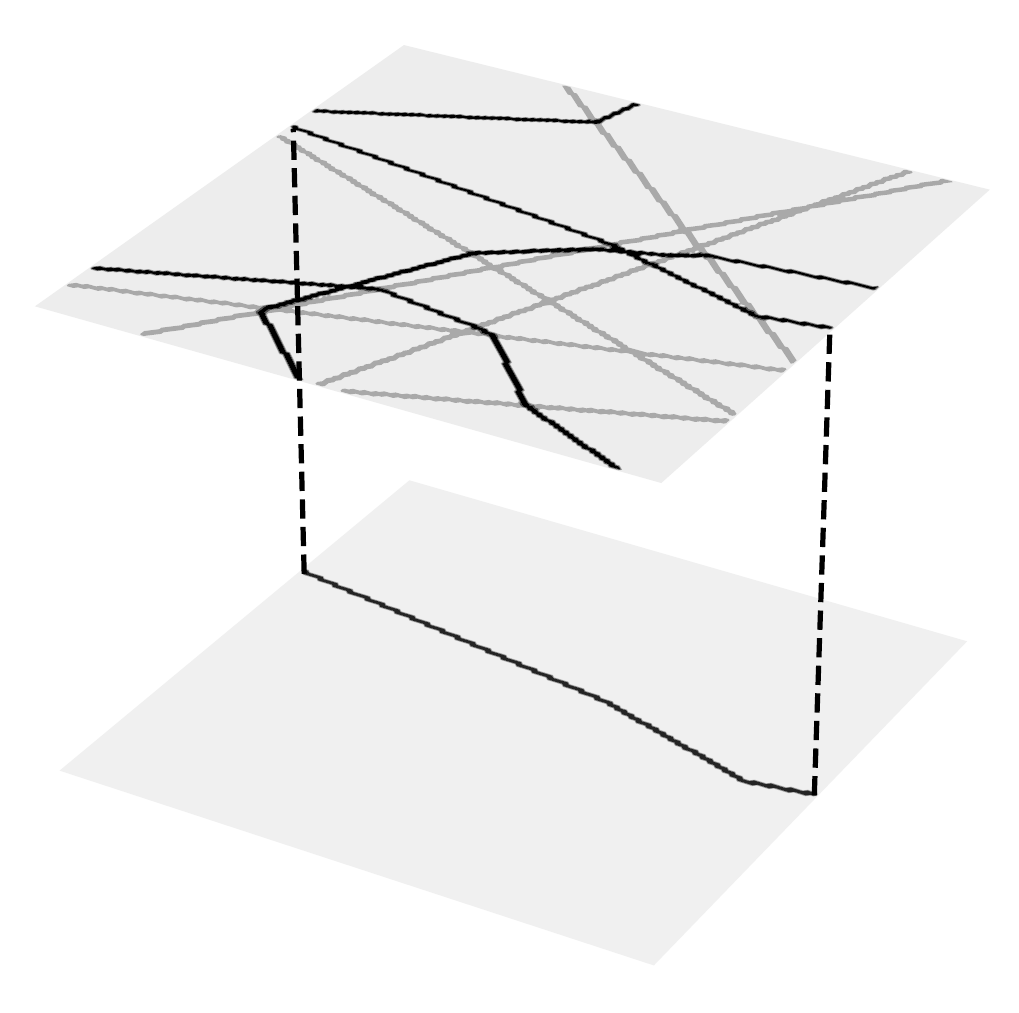}
    \end{minipage}\hfill
    \begin{minipage}{0.16\linewidth}
    \centering
    Unit 2\\
    \includegraphics[width=1\linewidth]{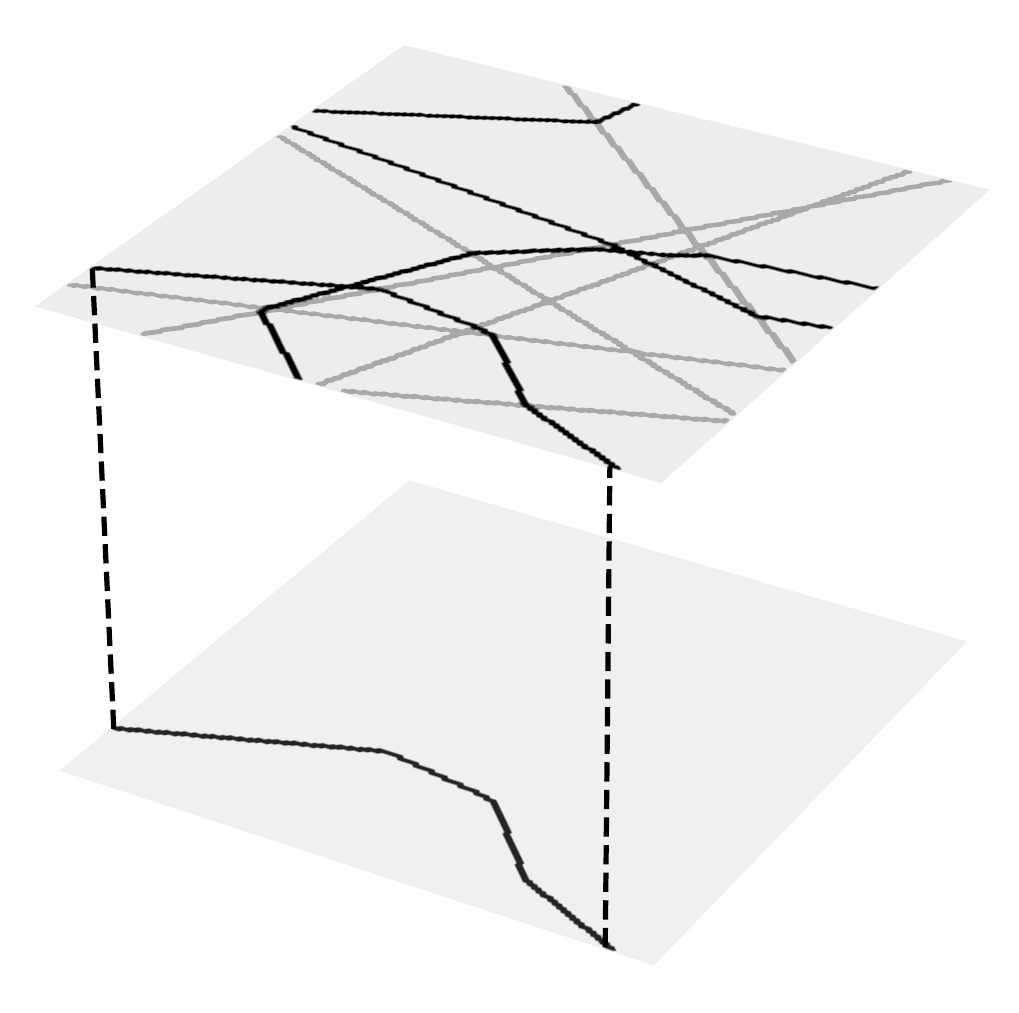}
    \end{minipage}\hfill
    \begin{minipage}{0.16\linewidth}
    \centering
    Unit 3\\
    \includegraphics[width=1\linewidth]{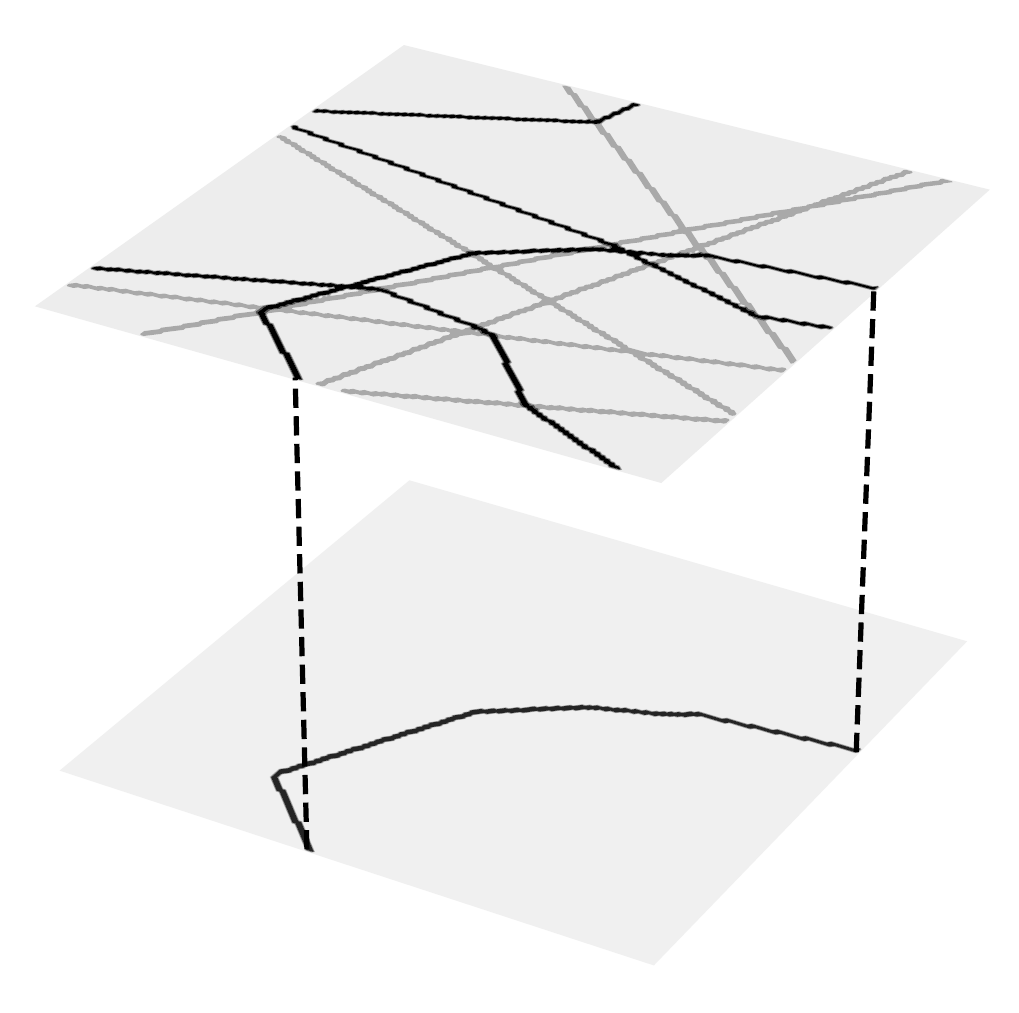}
    \end{minipage}\hfill
    \begin{minipage}{0.16\linewidth}
    \centering
    Unit 4\\
    \includegraphics[width=1\linewidth]{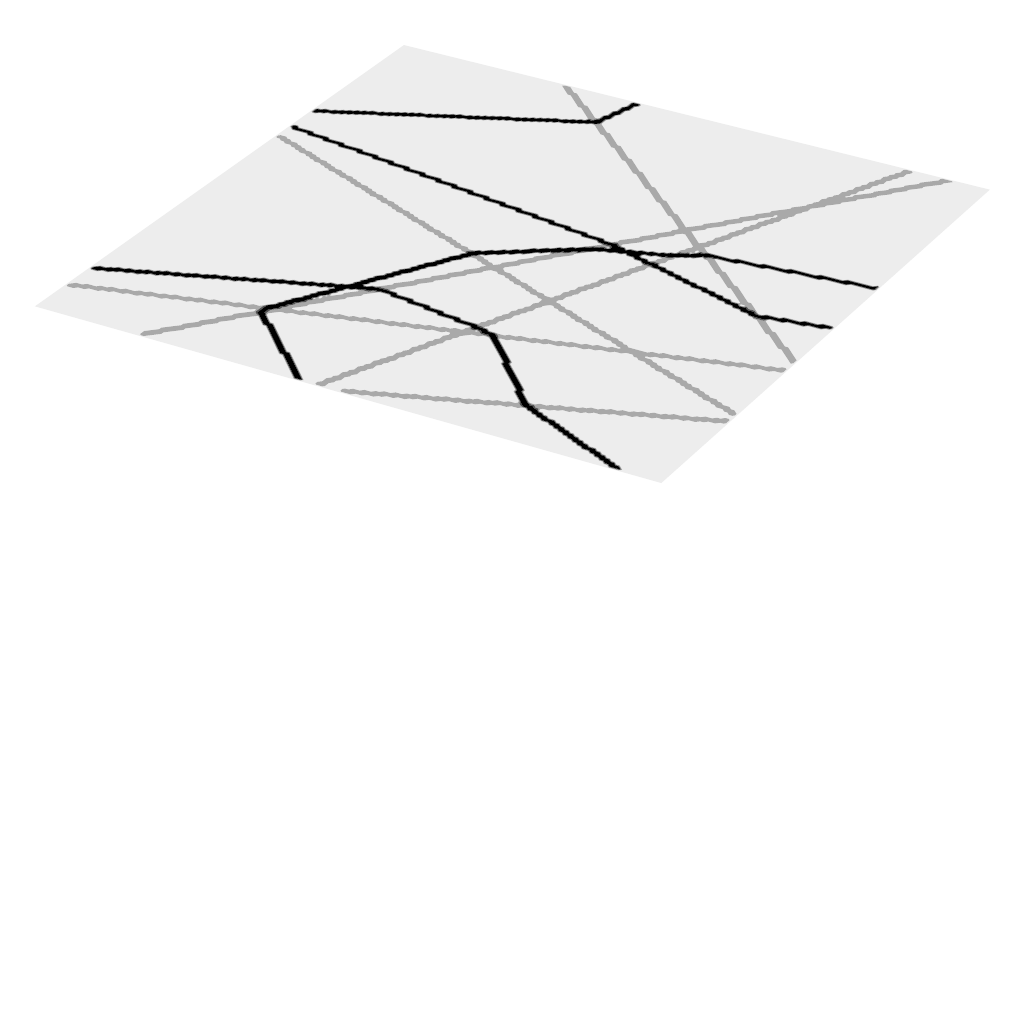}
    \end{minipage}\hfill
    \begin{minipage}{0.16\linewidth}
    \centering
    Unit 5\\
    \includegraphics[width=1\linewidth]{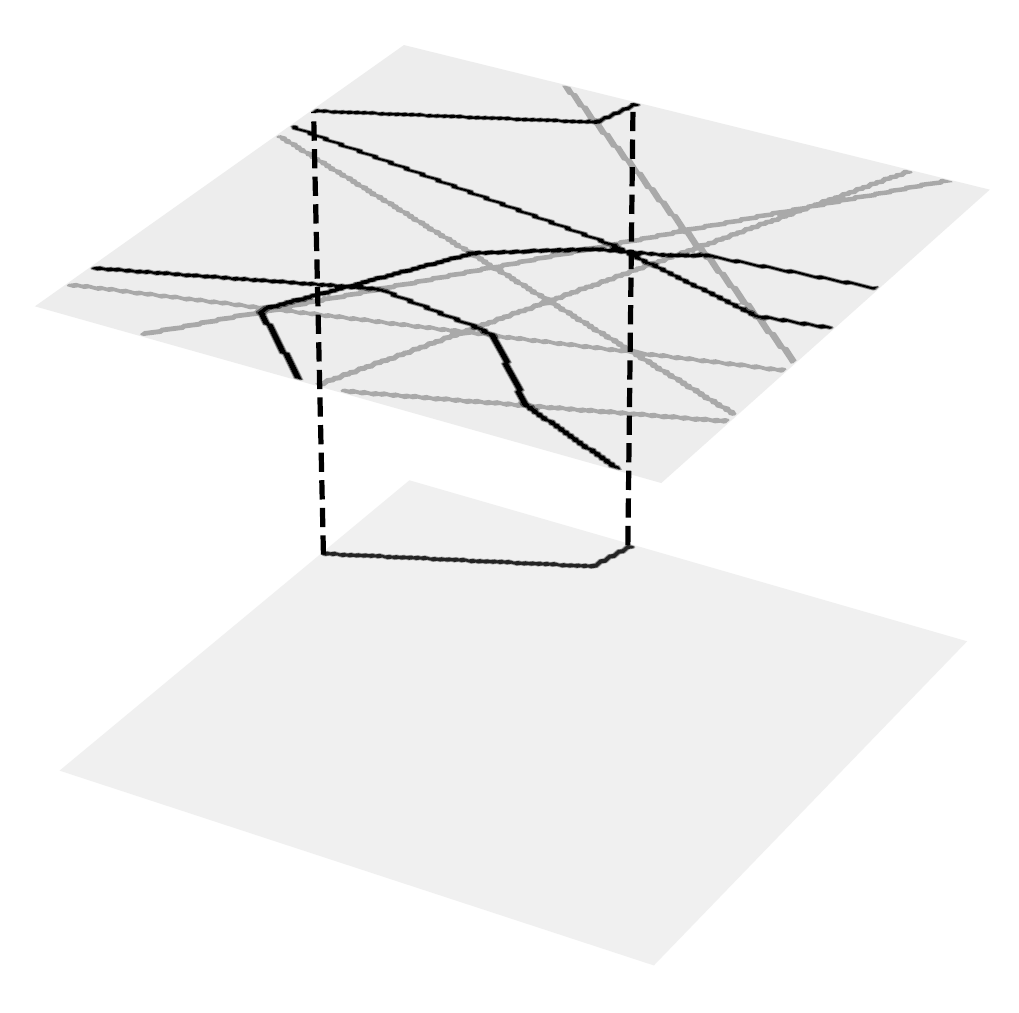}
    \end{minipage}\hfill
    \begin{minipage}{0.16\linewidth}
    \centering
    Unit 6\\
    \includegraphics[width=1\linewidth]{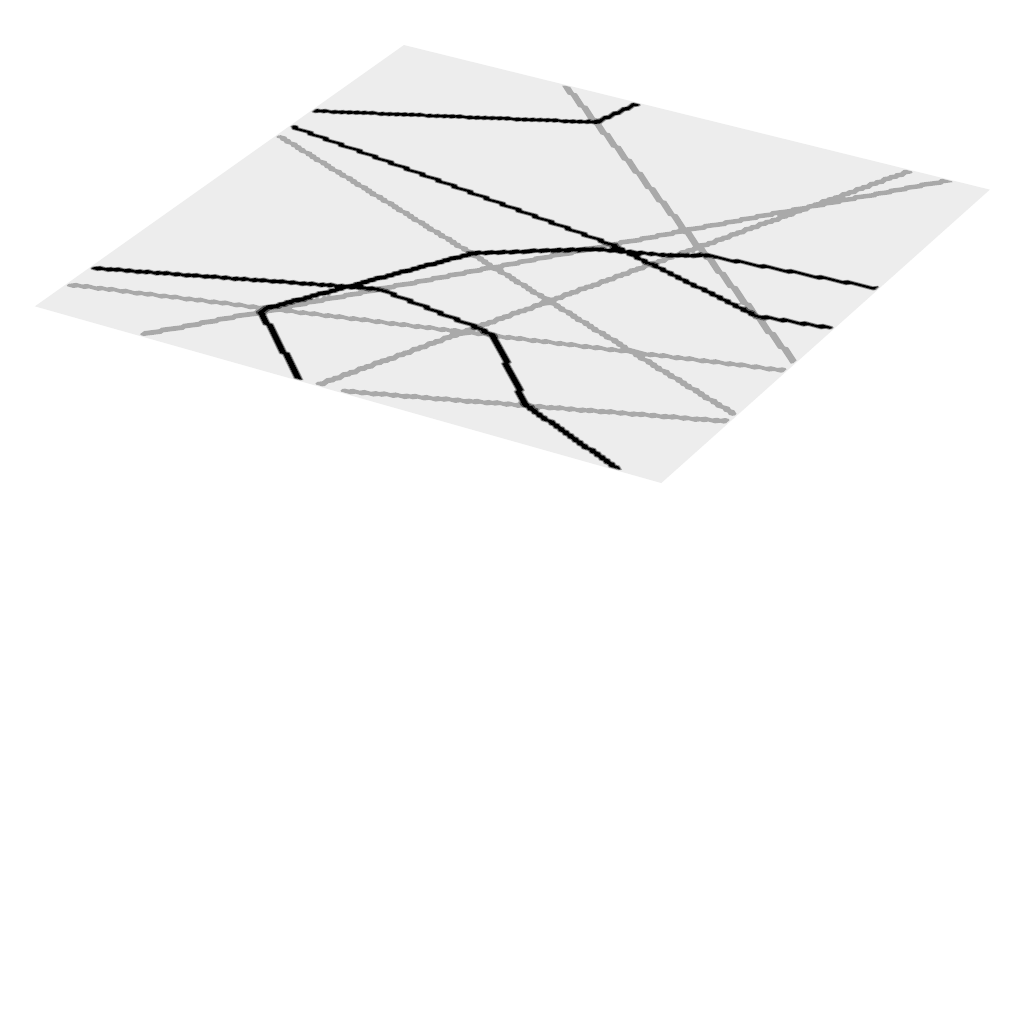}
    \end{minipage}\hfill\\
    \centering
    \begin{minipage}{0.16\linewidth}
    \centering
    Unit 1\\
    \includegraphics[width=1\linewidth]{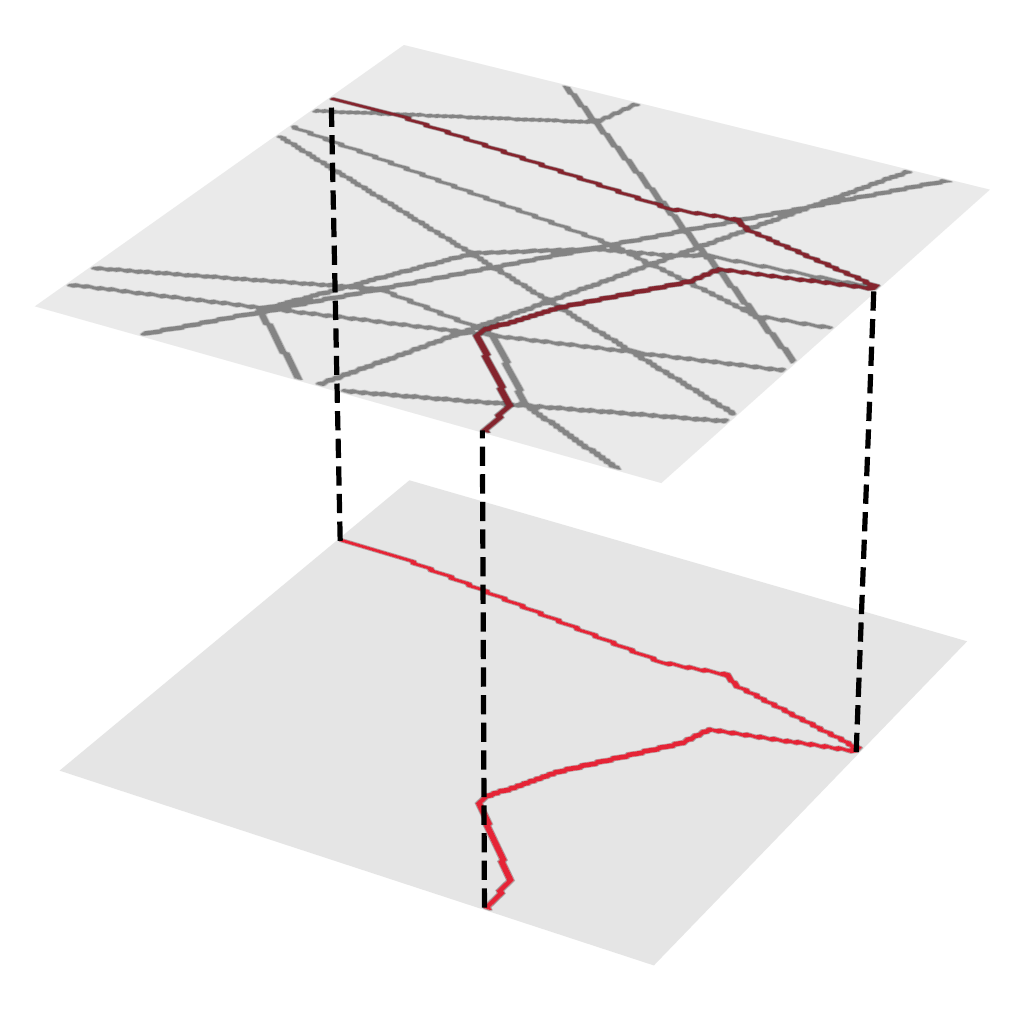}
    \end{minipage}\hfill\\
    \end{minipage}
    \caption{\small Additional depiction of the partitioning and subdivision happening layer after layer. Each unit also introduces a path in the input space which is depicted below the current partitioning with the highlighted path linked via a dotted line.\normalsize}
    \label{fig:additional_detail}
\end{figure}
\newpage

\begin{figure}[H]
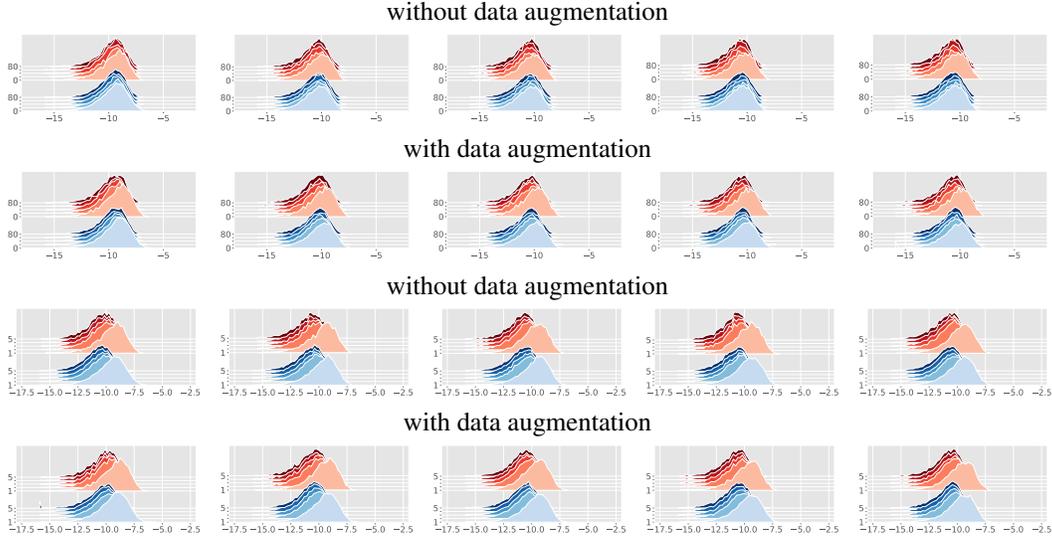

    \centering
    \begin{minipage}{1\linewidth}
    \centering
    without data augmentation
    \\
    \foreach \layer in {0,1,2,3,4}
        {
            \includegraphics[width=0.19\linewidth]{DISTANCES/loghistogram_epochs_save_test_v2_cifar10_False_l\layer.pdf}
        }
    \\
    with data augmentation
    \\
    \foreach \layer in {0,1,2,3,4}
        {
            \includegraphics[width=0.19\linewidth]{DISTANCES/loghistogram_epochs_save_test_v2_cifar10_True_l\layer.pdf}
        }
    \\
    without data augmentation
    \\
    \foreach \epoch in {0,20,40,60,80}
        {
            \includegraphics[width=0.19\linewidth]{DISTANCES/loghistogram_layers_save_test_v2_cifar10_False_e\epoch.pdf}
        }
    \\
    with data augmentation
    \\
    \foreach \epoch in {0,20,40,60,80}
        {
            \includegraphics[width=0.19\linewidth]{DISTANCES/loghistogram_layers_save_test_v2_cifar10_True_e\epoch.pdf}
        }
    
    \end{minipage}
    \caption{\small Additional depiction of the distances distribution.\normalsize}
    \label{fig:additional_detail3}
\end{figure}

\begin{figure}[H]
    \centering
    \hspace{0.1cm}\underline{\hspace{0.1cm}Layer 1\hspace{0.1cm}} \hspace{0.06cm}$\rightarrow$\hspace{0.06cm} \underline{\hspace{2.05cm}Each Layer $1$ region $q^{(1)}$ leads a different PD\hspace{2.05cm}}\hspace{0.08cm}
    \\
    \begin{minipage}{1\linewidth}
    \centering
    \includegraphics[width=0.1050\linewidth]{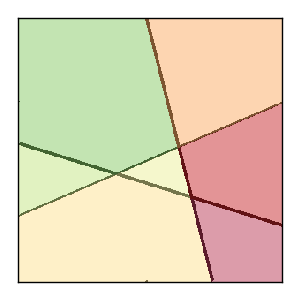}
    \includegraphics[width=0.1050\linewidth]{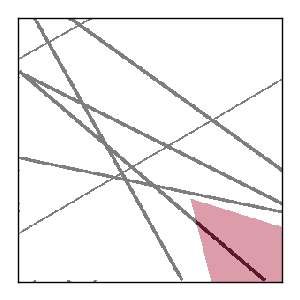}
    \includegraphics[width=0.1050\linewidth]{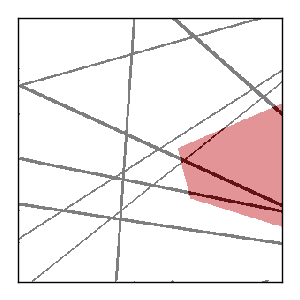}
    \includegraphics[width=0.1050\linewidth]{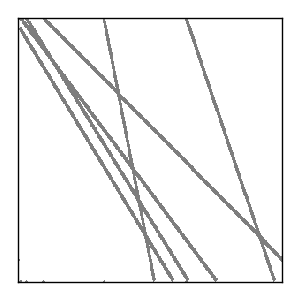}
    \includegraphics[width=0.1050\linewidth]{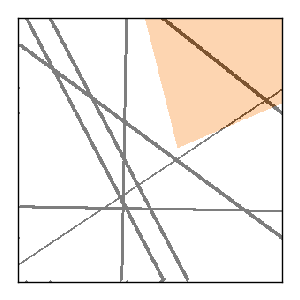}
    \includegraphics[width=0.1050\linewidth]{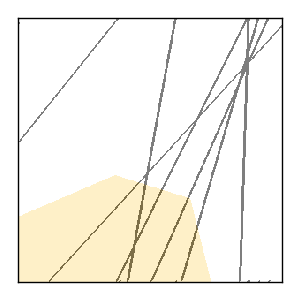}
    \includegraphics[width=0.1050\linewidth]{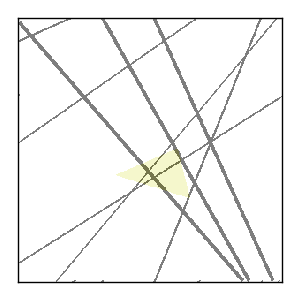}
    \includegraphics[width=0.1050\linewidth]{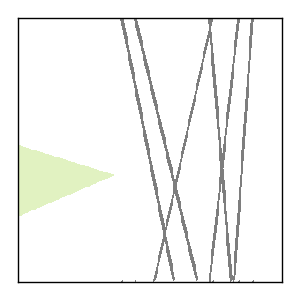}
    \includegraphics[width=0.1050\linewidth]{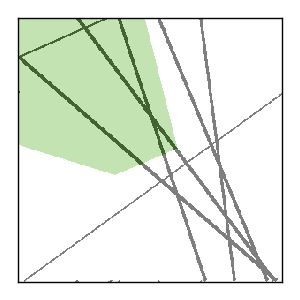}
    \end{minipage}\\
    \hspace{3.6cm}$\downarrow$\\
    \hspace{0.1cm}\underline{Layer 1 and 2} \hspace{0.06cm}$\leftarrow$\hspace{0.06cm} \underline{\hspace{2.05cm}Sub-division of each region with respective PD\hspace{2.05cm}}\hspace{0.08cm}
    \\
    \begin{minipage}{1\linewidth}
    \centering
    \includegraphics[width=0.1050\linewidth]{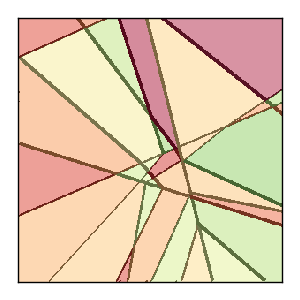}
    \includegraphics[width=0.1050\linewidth]{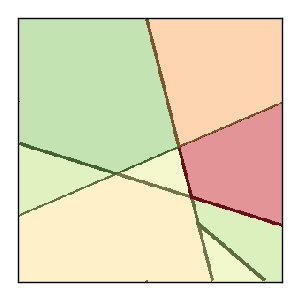}
    \includegraphics[width=0.1050\linewidth]{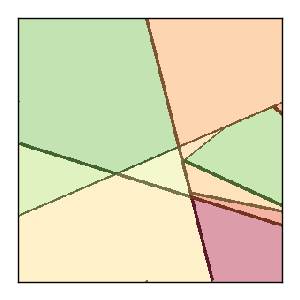}
    \includegraphics[width=0.1050\linewidth]{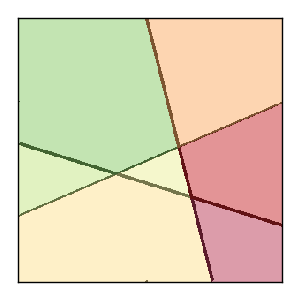}
    \includegraphics[width=0.1050\linewidth]{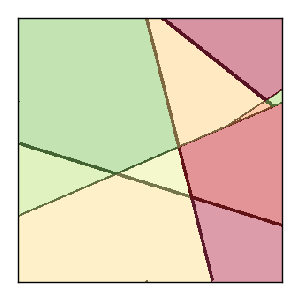}
    \includegraphics[width=0.1050\linewidth]{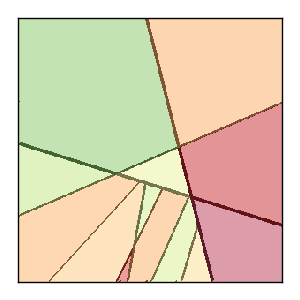}
    \includegraphics[width=0.1050\linewidth]{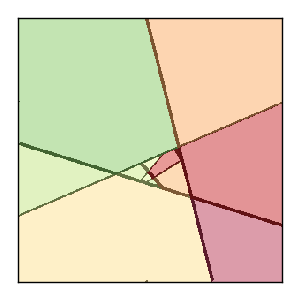}
    \includegraphics[width=0.1050\linewidth]{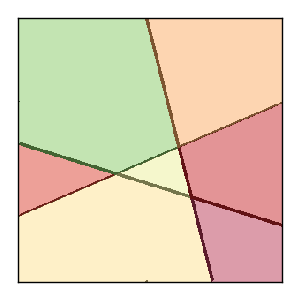}
    \includegraphics[width=0.1050\linewidth]{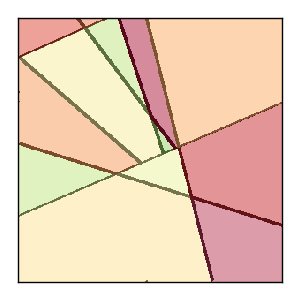}
    \end{minipage}
    \caption{\small Visual depiction of Cor.~\ref{cor:2layer} for a $2$-layer DN with $3$ units at the first layer (leading to $4$ regions) and $8$ units at the second layer with random weights and biases. The colors are the DN input space partitioning w.r.t. the first layer. Then for each color (or region) the layer1-layer2 defines a specific Power Diagram that will sub-divide this aforementioned region (this is the first row) where the region is colored and the Power Diagram is depicted for the whole input space. Then this sub-division is applied onto the first layer region only as it only sub-divides its region (this is the second row on the right). And finally grouping together this process for each of the $4$ region, we obtain the layer-layer 2 space partitioning (second row on the left).\normalsize}
    \label{fig:2layer3unit_partitioning}
\end{figure}

\section{Proofs}

\subsection{Proof of Lemma~\ref{prop:vertical_unit}: Single Unit Projection}
\label{proof:1}
Follows from \cite{johnson1960advanced} that demonstrates that the boundaries in the input space $\mathcal{X}$ defining the regions of the unit PD are the vertical projections of the polytope ($\mathcal{P}_k$) face intersections defined as $\Epsilon_{k,r}(\mathcal{X}) \cap \Epsilon_{k,r*}(\mathcal{X})$ for neighbouring faces $r$ and $r^*$.

\subsection{Proof of Lemma~\ref{lemma:vertical_layer}: Single Layer Projection}
\label{proof:2}
Follows from the previous section in which it is demonstrated that the boundaries of a single unit PD is obtained by vertical projection of the polytope edges. In the layer case, the edges of $\P$ correspond to all the points in the input space $\mathcal{X}$ s.t. $\bz$ belongs to an edge of at least one of the polytopes $\mathcal{P}_k,\forall k$ making up $\P$. The layer PD having for boundaries the union of all the per unit PD boundaries, it follows directly that the vertical projection of the edges of $\P$ form the layer PD boundaries.

\subsection{Proof Complexity}
\label{proof:layer_complexity}
Recalling Section~\ref{sec:back}, a layer $\ell$ MASO produces its output by: first inferring in which cell $\br(\bz^{(\ell-1)}(\bx))$ lies in the layer PD partitioning from Theorem~\ref{thm:layer_PD}; and then affinely transforming the input via $A^{(\ell)}_{\br(\bz^{(\ell-1)}(\bx))} \bz^{(\ell-1)} + B^{(\ell)}_{\br(\bz^{(\ell-1)}(\bx))}$. 
The inference problem of determining in which power diagram cell an input $\bx$ falls 
%

\subsection{Proof of Theorem~\ref{thm:unitPD}: Single Unit Power Diagram}
\label{sec:proof_one_layer}

Let first consider the case of $2$ units.

\begin{lemma}
\label{lemma:2units}
The layer input space of the $[1^{th},2^{th}]$-MASO units at layer $l$ is a weighted Voronoi Diagram with a maximum of $R^{\ell} \times R^{\ell}$
regions, centroids 
$
\bA \{[t]_{1},[t]_{2}\} =    [\bA]_{1,[t]_{1},\bigcdot} + [\bA]_{2,[t]_{2},\bigcdot}$, and biases $\bb \{ [t]_{1},[t]_{2} \} = [\bb]_{1,[t]_{1}}+ [\bb]_{2,[t]_{2}} -2 \langle [\bA]_{1,[t]_{1},\bigcdot} , [\bA]_{2,[t]_{2},\bigcdot} \rangle$.
\end{lemma}



\small
\begin{proof}
\begin{align*}
 \mathcal{V}(&[t]_{1},[t]_{2})=  \mathcal{V}([t]_{1})  \cap \mathcal{V}([t]_{2})\\
& =  \{ \bx \in \mathbb{R}^{D}|\argmax_{i} \langle \bx,[\bA]_{1,i,\bigcdot}  \rangle +[B]_{1,i} = [t]_{1} \}  \cap \{ \bx \in \mathbb{R}^{D}|\argmax_{j} \langle \bx,[\bA]_{2,j,\bigcdot}  \rangle +[B]_{2,j} = [t]_{2} \}\\
& =  \{ \bx \in \mathbb{R}^{D}| \argmin\limits_{i} \left \| \bx- [\bA]_{1,i,\bigcdot} \right \|^2 + [\bb]_{1,i} = [t]_{1} \}  \cap  \{ \bx \in \mathbb{R}^{D}| \argmin\limits_{j} \left \| \bx- [\bA]_{2,j,\bigcdot} \right \|^2 + [\bb]_{2,j} = [t]_{2} \} \nonumber \\
& =  \{ \bx \in \mathbb{R}^{D}| \argmin\limits_{i,j} \left \| \bx- [\bA]_{1,i,\bigcdot} \right \|^2 + [\bb]_{1,i}   +  \left \| \bx- [\bA]_{2,j,\bigcdot} \right \|^2 + [\bb]_{2,j} = ([t]_{1},[t]_{2}) \} \\ \nonumber
& = \{ \bx \in \mathbb{R}^{D}| \argmin\limits_{i,j} 2\left \| \bx \right \|^{2} + \left \| [\bA]_{1,i,\bigcdot} \right \|^2  -2\langle  \bx , [\bA]_{1,i,\bigcdot}  \rangle  + [\bb]_{1,i}
   + \left \| [\bA]_{2,j,\bigcdot} \right \|^2 -2 \langle  \bx ,  [\bA]_{2,j,\bigcdot} \rangle + [\bb]_{2,j} = ([t]_{1},[t]_{2}) \} \\ \nonumber
& = \{ \bx \in \mathbb{R}^{D}| \argmin\limits_{i,j} 2\left \| \bx \right \|^{2}    -2\langle  \bx , [\bA]_{1,i,\bigcdot}+[\bA]_{2,j,\bigcdot}  \rangle  + \left \| [\bA]_{1,i,\bigcdot} \right \|^2+\left \| [\bA]_{2,j,\bigcdot} \right \|^2 + [\bb]_{1,i}
    + [\bb]_{2,j} = ([t]_{1},[t]_{2}) \} \\ \nonumber
& = \{ \bx \in \mathbb{R}^{D}| \argmin\limits_{i,j} \left \| \bx -([\bA]_{1,i,\bigcdot} + [\bA]_{2,j,\bigcdot}  ) \right \|^2 +  [\bb]_{1,i}+ [\bb]_{2,j} -2 \langle [\bA]_{1,i,\bigcdot} , [\bA]_{2,j,\bigcdot} \rangle  = ([t]_{1},[t]_{2})  \} \\ \nonumber
&  = \{ \bx \in \mathbb{R}^{D}| \argmin\limits_{i,j} \left \| \bx -\bA \{ i,j \}  \right \|^2 + \bb \{ i,j \} = ([t]_{1},[t]_{2})  \},
\end{align*}
where, $ \bb \{i,j \} =\| [\bA]_{1,i,\bigcdot}\|^2+2[B]_{1,i}+ \| [\bA]_{2,j,\bigcdot}\|^2+2[B]_{2,j} +2 \langle [\bA]_{1,i,\bigcdot} , [\bA]_{2,j,\bigcdot} \rangle $ and, $\bA \{ i,j \} =    [\bA]_{1,i,\bigcdot} + [\bA]_{2,j,\bigcdot}$.
\\

\end{proof}
\normalsize

The $D$ units case: apply recursively Lemma \ref{lemma:2units}.

\subsection{Proof of Theorem~\ref{thm:layerPD}: Single Layer Power Diagram}
\label{appendix:proof_affine}
We first derive a preliminary result in which a layer follows an affine transformation to then generalize by considering how each region of the previously built partitioning transforms the inputs lying in it linearly.

\paragraph{Input Space Partitioning of a Single Layer Following an Affine Transform}

Consider a layer with input an affine transformation of $\bx \in \mathcal{X}$ as $G\bx+\bh$ where $G$ is an arbitrary matrix and $\bh$ an arbitrary vector. We consider this affine transformation as a linear DN layer. We now express the layer PD w.r.t. the input space as $\Omega^{(2)}(\mathcal{X}^{(0)})$.
Define the centroids and radius
\begin{align}
    \mu^{(1\leftarrow 2)}_{\br^{(2)}} = & G^\top \sum_{k=1}^{D^{(2)}} [A^{(2)}]_{k,[\br^{(2)}]_{k},\bigcdot}=G^\top \mu^{(2)}_{\br^{(2)}}\\
    {\rm rad}^{(1\leftarrow 2)}_{\br^{(2)}} = &-\| G^\top \mu^{(1\leftarrow 2)}_{\br^{(2)}}\|^2-2 {\mu^{(1\leftarrow 2)}_{\br^{(\ell)}}}^\top \bh-2\sum_{k=1}^{D^{(2)}} [B^{(2)}]_{k,[\br]_{k}}
\end{align}
where $\mu^{(2)}_{\br^{(2)}}$ is as defined in Theorem~\ref{thm:layerPD}.
\begin{lemma}
\label{thm:layer_affine}
The input space partitioning of a 2-layer DN with first layer linear is given by \\$\Omega^{(1, 2)}(\mathcal{X}^{(0)})=\PD(\mathcal{X};\{(\mu_{\br},{\rm rad}_{\br}),\forall \br \})$.
\end{lemma}

\begin{lemma}
\label{lemma:2layers}
The layer input space of the $[1^{th},2^{th}]$-MASO units at layer $l$ is a weighted Voronoi Diagram with a maximum of $R^{\ell} \times R^{\ell}$
regions, centroids 
$
\bA \{[t]_{1},[t]_{2}\} =   G^T [\bA]_{1,[t]_{1},\bigcdot} + G^T[\bA]_{2,[t]_{2},\bigcdot}$, and biases $\bb \{[t]_{1},[t]_{2} \} = [\bb]_{1,[t]_{1}}+ [\bb]_{2,[t]_{2}} -2 \langle G^T [\bA]_{1,[t]_{1},\bigcdot} , G^T [\bA]_{2,[t]_{2},\bigcdot} \rangle$.
\end{lemma}

\small
\begin{proof}
\begin{align*}
 \mathcal{V}(&[t]_1,[t]_2)=  \mathcal{V}([t]_{1})  \cap \mathcal{V}([t]_{2})\\
& =  \{ \bx \in \mathbb{R}^{D}|\argmax_{i} \langle G\bx+h,[\bA]_{1,i,\bigcdot}  \rangle +[B]_{1,i} = [t]_{1} \}  \cap \{ \bx \in \mathbb{R}^{D}|\argmax_{j} \langle G\bx+h,[\bA]_{2,j,\bigcdot}  \rangle +[B]_{2,j} = [t]_{2} \}\\
& =  \{ \bx \in \mathbb{R}^{D}|\argmax_{i} \langle \bx,G^T [\bA]_{1,i,\bigcdot}  \rangle +\langle [\bA]_{1,i,\bigcdot},h\rangle+[B]_{1,i} = [t]_{1} \}  \cap \{ \bx \in \mathbb{R}^{D}|\argmax_{j} \langle \bx,G^T [\bA]_{2,j,\bigcdot}  \rangle +\langle [\bA]_{2,j,\bigcdot},h\rangle +[B]_{2,j} = [t]_{2} \}\\
& =  \{ \bx \in \mathbb{R}^{D}| \argmin\limits_{i} \left \| \bx- G^T [\bA]_{1,i,\bigcdot} \right \|^2 + [\bb]_{1,i} = [t]_{1} \}  \cap  \{ \bx \in \mathbb{R}^{D}| \argmin\limits_{j} \left \| \bx- G^T [\bA]_{2,j,\bigcdot} \right \|^2 + [\bb]_{2,j} = [t]_{2} \} \nonumber \\
& =  \{ \bx \in \mathbb{R}^{D}| \argmin\limits_{i,j} \left \| \bx- G^T [\bA]_{1,i,\bigcdot} \right \|^2 + [\bb]_{1,i}   +  \left \| \bx- G^T [\bA]_{2,j,\bigcdot} \right \|^2 + [\bb]_{2,j} = ([t]_{1},[t]_{2}) \} \\ \nonumber
& = \{ \bx \in \mathbb{R}^{D}| \argmin\limits_{i,j} 2\left \| \bx \right \|^{2} + \left \| G^T [\bA]_{1,i,\bigcdot} \right \|^2  -2\langle  \bx , G^T [\bA]_{1,i,\bigcdot}  \rangle  + [\bb]_{1,i}
   + \left \| G^T [\bA]_{2,j,\bigcdot} \right \|^2 -2 \langle  \bx ,  G^T [\bA]_{2,j,\bigcdot} \rangle + [\bb]_{2,j} = ([t]_{1},[t]_{2}) \} \\ \nonumber
& = \{ \bx \in \mathbb{R}^{D}| \argmin\limits_{i,j} 2\left \| \bx \right \|^{2}    -2\langle  \bx , G^T [\bA]_{1,i,\bigcdot}+G^T [\bA]_{2,j,\bigcdot}  \rangle  + \left \| G^T [\bA]_{1,i,\bigcdot} \right \|^2+\left \| G^T [\bA]_{2,j,\bigcdot} \right \|^2 + [\bb]_{1,i}
    + [\bb]_{2,j} = ([t]_{1},[t]_{2}) \} \\ \nonumber
& = \{ \bx \in \mathbb{R}^{D}| \argmin\limits_{i,j} \left \| \bx -(G^T [\bA]_{1,i,\bigcdot} + G^T [\bA]_{2,j,\bigcdot}  ) \right \|^2 +  [\bb]_{1,i}+ [\bb]_{2,j} -2 \langle G^T [\bA]_{1,i,\bigcdot} , G^T [\bA]_{2,j,\bigcdot} \rangle  = ([t]_{1},[t]_{2})  \} \\ \nonumber
&  = \{ \bx \in \mathbb{R}^{D}| \argmin\limits_{i,j} \left \| \bx -G^T\bA \{i,j\}  \right \|^2 + \bb \{i,j \} = ([t]_{1},[t]_{2})  \},
\end{align*}
where, $ \bb\{i,j\} = [\bb]_{1,i}+ [\bb]_{2,j} -2 \langle G^T [\bA]_{1,i,\bigcdot} , G^T [\bA]_{2,j,\bigcdot} \rangle $ and, $\bA \{i,j \} =    [\bA]_{1,i,\bigcdot} + [\bA]_{2,j,\bigcdot}$\\and $[\bb]_{1,i}=-\|G^T [\bA]_{1,i,\bigcdot}\|^2-2\langle [\bA]_{1,i,\bigcdot},h \rangle-2[B]_{1,i}$. 
\\
We thus have $b\{i,j\}=-2B\{i,j\}-2A\{i,j\}^Th-\|G^TA\{i,j\} \|^2$.
\end{proof}
\normalsize


\end{document}